\documentclass[twoside]{article}
\usepackage[accepted]{aistats2016}

\usepackage{float}
\usepackage{epsf}
\usepackage{graphicx}
\usepackage{amsmath}
\usepackage{amssymb}
\usepackage{algorithmic}
\usepackage{algorithm}
\usepackage{amsthm}

\newtheorem{theorem}{Theorem}[section]
\newtheorem{lemma}{Lemma}[section]
\newtheorem{assumption}{Assumption}[section]

\begin{document}

%

%

\twocolumn[

\aistatstitle{Learning relationships between data obtained independently}

\aistatsauthor{Alexandra Carpentier \And Teresa Schl\"uter}

\aistatsaddress{ Universit\"at Potsdam \And London School of Economics and WBGU} ]

\begin{abstract}
The aim of this paper is to provide a new method for learning the relationships between data that have been obtained independently. Unlike existing methods like matching, the proposed technique does not require any contextual information, provided that the dependency between the variables of interest is monotone. It can therefore be easily combined with matching in order to exploit the advantages of both methods. This technique can be described as a mix between quantile matching, and deconvolution. We provide for it a theoretical and an empirical validation.
\end{abstract}

\vspace{-0.3cm}

\section{Introduction}
\vspace{-0.3cm}

The Big Data phenomenon is made possible by the parallelisation of data acquisition. The data are not collected by a centralised organism, but by several contributors. This is a strength since it allows a huge variety and quantity of data to be made available. However, a necessary step for studying data from different contributors is \textit{merging} these datasets. For instance, consider the classical problem of knowing which part $Y$ of its income $X$ an individual spends on housing. Formally, we model the relation between $X$ and $Y$ by 
$$Y = h(X,Z) + \epsilon,$$
where $\epsilon$ is a noise, $Z$ is (in case they are provided) some additional \textit{contextual variables} (like e.g.~age, sex, job, etc) and $h_Z := h(.,Z)$ is the dependence function between $Y$ and $X$ given the contextual variables $Z$. The objective is to estimate the dependence relation $h_Z$. An obstacle for answering this question is that in most countries, the data on wages $\mathbf X$ are collected by one kind of agent (e.g.~the office for national statistics), and the data on housing transactions $\mathbf Y$ are collected by another kind of agent (e.g.~financial institutions specialized in mortgage lending). The dependency $h_Z$ between these variables cannot be established immediately, since the two data sets have been collected independently : $\mathbf X$ and $\mathbf Y$ are independent. Standard results on regression don't apply in this context. Estimating $h_Z$, or \textit{merging the variables $\mathbf X$ and $\mathbf Y$}, is then challenging. This kind of problem becomes more and more common with the growing importance of social networks such as twitter, facebook, etc. For instance, it is desirable to combine the data collected by these social networks by merging the user profiles, in the interest of a fuller analysis of their content.

A very popular method for overcoming this problem is called \textit{matching}, see~\cite{walter1999matching,monge1996field, cohen2002learning}. Suppose that in addition to collecting data on incomes and house prices, the agents also collect \textit{contextual data} (corresponding to the contextual variables $Z$) such as the age, sex, job, etc. In other words, the two independent datasets are of the form $(\mathbf X, \mathbf Z(1))$ and $(\mathbf Y, \mathbf Z(2))$ where $\mathbf Z(1)$ and $\mathbf Z(2)$ are the contextual data of respectively $\mathbf X$ and $\mathbf Y$. The matching procedure consists in associating the data in $\mathbf X$ with the data in $\mathbf Y$ that are most similar in terms of the contextual data, i.e.~in associating the $(X_i, Z(1)_i)$ with the $(Y_j, Z(2)_j)$ such that a given distance between $Z(1)_i$ and $Z(2)_j$ is minimised. Once the matched dataset is available, $h_Z$ can be inferred using a suitable regression method. Also using this idea of cleverly taking advantage of contextual variables, many other methods have been considered for dealing with the problems of fusing data coming from independent sources, as \textit{data fusion, linkage, data integration, etc}~\cite{auer2013introduction,dong2009data,bleiholder2008data, dong2013big}.

These approaches are very reasonable, but they rely heavily on the contextual data $\mathbf Z(1)$ and $\mathbf Z(2)$. If these contextual data are not very detailed, these methods can perform poorly. For instance, the most complete UK database on housing transactions which is collected by the Land Registry contains the house market transaction prices, some structural characteristics of the house and the geographical location, but no information on the buyer. On the other hand, the micro databases that contain information on the wage of individuals have few if no information on the type of lodging that the individuals occupy, apart from sometimes an approximate geographic location. In this baseline problem, the approximate geographic location is the only matching variable available. In some highly populated area, there will be many matches. In this situation, an additional noise, or even a model misspecification will be introduced by the imprecision of the matching. In the case of the social network example, the collected data are often partially anonymized, which makes the merging process difficult.


In this paper, we develop an alternative method for learning the relationships between two variables that have been collected independently. Our approach \textit{does not rely on contextual variables} and can therefore be used to improve on any given matching method - it is particularly interesting in the situation where many possible matches are available. It can even be used when \textit{no contextual variables are available}, i.e.~when the only available data are $\mathbf X$ and $\mathbf Y$ and when they have been collected independently.\\
The necessary assumptions in order for our method to work is that the dependence function $h_Z$ given the contextual variables $Z$ is monotone (increasing or decreasing) and that the noise distribution is known (exactly or by an estimate). In the income and house price example, it is clear that $h_Z$ is an increasing function given standard contextual variables $Z$ such as sex, age, job, etc - the larger the income of an individual, the more it will pay for its house \textit{on average} 
given the standard contextual variables. In the social network example, a high utilisation of a given social mainstream network is \textit{on average} positively connected with a high utilisation of another mainstream social media. This situation - an increasing relation $h_Z$ given the contextual variables $Z$ - is in fact fairly common, since we only need \textit{averaged} monotonicity (through $h$).\\
The fact that our method can even be used when no contextual variables are available can seem very surprising, since it is counter intuitive that the relation $h$ (when there are no contextual variables, $h_Z = h$) between $\mathbf X$ and $\mathbf Y$ can be deduced from data $\mathbf X$ and $\mathbf Y$ that are independent. The reason why this is possible is the monotonicity of $h$. In the simple case when the noise $\epsilon$ is null, if it is known that $h$ is e.g.~increasing, then the $t$ percentile of $\mathbf X$ corresponds to the $t$ percentile of $\mathbf Y$. The relation $h$ between $\mathbf X$ and $\mathbf Y$ can then be inferred by matching the quantiles of the distributions of $\mathbf X$ and $\mathbf Y$. When there is noise, the problem is slightly more involved, and this is the situation that we develop in this paper. Our approach mixes elements of \textit{quantile regression and quantile comparison}, and of \textit{deconvolution}, applied in a non-standard way to the problem of merging data sets. Quantile regression and quantile comparison (see e.g.~\cite{lombard2005nonparametric, wilcox2012comparing, einmahl1999confidence, doksum1976plotting}) consists in fitting a distribution with other distributions, or comparing a distribution with other distributions, by using the quantiles of the distributions. Distribution deconvolution (see e.g.~\cite{dattner2011deconvolution,chan1998total, fan1991optimal, cordy1997deconvolution, carroll1988optimal, bell1995information, chan1998total,moulines1997maximum, starck2003wavelets}) consists in estimating the distribution of a random variable using noisy samples.


This paper is structured as follows. In Section~\ref{s:setting}, we describe the setting of the paper. In Section~\ref{s:results}, we provide the estimator of $h_Z$, and associated bounds on its performance. Finally, we present numerical experiments on world bank data, and on census and land registry data in Section~\ref{s:expes} for assessing the practical impact of our method. The Appendix contains the proofs of the main theorems and additional experiments. 

\vspace{-0.3cm}

\section{Setting}\label{s:setting}
\vspace{-0.3cm}

In this section, we present in a formal way the setting and the objective. 




\subsection{Presentation of the model}
\vspace{-0.3cm}

Let $d \in \mathbb N$ be the dimension of the contextual variables $Z$. In our approach, we allow $d=0$, i.e.~no contextual variables are available. We do not make assumptions on the way the contextual variables are generated. 

We assume the following modelling for the explicative variables $X$:
$$X|Z \sim f_Z,$$
where $f_Z$ is a density on $\mathbb R$ (and $F_Z$ is the associated distribution function). We now assume the following modelling for the explained variable $Y$:
$$Y := h(X, Z) + \epsilon,$$
where $\epsilon|Z \sim \xi_Z$ is a noise that is independent of $X$, where $\xi_Z$ is a density on $\mathbb R$ of mean $0$, and where for any $Z$, $h(.,Z)$ is a function from $\mathbb R$ to $\mathbb R$ which is \textit{monotone}. We assume that we know the distribution $\xi_Z$. We write $f_{h,Z}$ for the density of $h(X,Z)|Z$ (and $F_{h,Z}$ for the associated distribution function), and $g_Z$ for the density of $Y|Z$. Given the model, we have
$$g_Z = f_{h,Z} * \xi_Z,$$
i.e.~$g_Z$ is the convolution of $f_{h,Z}$ and $\xi_Z$.


\noindent
{\bf Remark on the monotonicity assumption:}  Assuming that for any $Z$, $h(.,Z)$ is monotone, e.g.~non-decreasing, means that given the contextual variables $Z$ and \textit{on average}, a larger $X$ corresponds to a larger $Y$. In the rent and wage example, it is a very  reasonable assumption that richer people, given standard contextual variables as their sex, age, socio professional category, location, etc spend \textit{on average} more on lodging. Another example concerns the number of followers on two social media such as e.g.~twitter and facebook. It is reasonable to assume that someone who is very active on facebook is more likely to be active on twitter. This assumption is not very restrictive since it is made \textit{on average}, and we do not make the assumption that there is a strict order relationship that holds for every individual, which is a much stronger assumption.


\vspace{-0.3cm}

\subsection{Data}
\vspace{-0.3cm}

We assume that we are given two databases:
$$(\mathbf X, \mathbf Z(1)) \quad \mathrm{and}\quad (\mathbf Y, \mathbf Z(2))$$
with respectively $m$ and $n$ individuals, where $X_i|Z(1)_i \sim f_{Z(1)_i}$ and $Y_i|Z(2)_i \sim g_{Z(2)_i}$ and are \textit{totally independent}, as e.g.~in the case where they are collected from two independent sources. For instance, $\mathbf X$ can be a dataset containing a sample of wages of individual in a geographical unit, and $\mathbf Y$ a dataset containing a sample of house prices in the same geographical unit, and $\mathbf Z(1), \mathbf Z(2)$ can be standard categorical variables such age age and sex. The problem here is that these datasets have been collected independently of each other, and one does not know which individual, of a given wage, buys which house.


\noindent
\textbf{Objective: Infer the function $h$ from the data $(\mathbf X, \mathbf Z(1)) \quad \mathrm{and}\quad (\mathbf Y, \mathbf Z(2))$.}


\vspace{-0.3cm}

\subsection{Preliminary matching procedure}\label{ss:match}
\vspace{-0.3cm}



Matching procedures (see e.g.~\cite{walter1999matching,monge1996field, cohen2002learning}) aim at merging $(\mathbf X, \mathbf Z(1))$ and $(\mathbf Y, \mathbf Z(2))$ by finding good matches between the contextual variables $\mathbf Z(1)$ and $\mathbf Z(2)$. The basic procedure can be summarized as follows. 
Let $d$ be a distance function between the data points in $ \mathbf Z(1)$ and the data points in $ \mathbf Z(2)$. A non-robust matching procedure aims at associating any $(X_i, Z(1)_i)$ of the first dataset with the point $(Y_j, Z(2)_j)$ of the second dataset that minimises $d(Z(1)_i,Z(2)_j)$. A more robust generalisation of this method, related to nearest neighbours methods, is to match, for any $z$, the points $\mathcal D_1(z)$ of the first dataset whose contextual variables $\mathbf Z(1)$ are $\upsilon-$close to $z$, with points $\mathcal D_2(z)$ of the second dataset whose contextual variables $\mathbf Z(2)$ are $\upsilon-$close to $z$, for $\upsilon>0$.

Particularly in the case of not very detailed contextual variables (age, sex, etc), a typical matching approach would then not return a one to one match from $\mathbf X$ to $\mathbf Y$, but would map subsets of $\mathbf X$ to subsets of $\mathbf Y$, in function of $Z$. So a matching procedure outputs, for values of $Z$ that exist in the dataset, the following subsets of $\mathbf X$ and $\mathbf Y$ that correspond to $Z$ :
$$\mathbf X^Z \quad \mathrm{and} \quad \mathbf Y^Z.$$
This provides a first merging of the variables, but in the case where for given $Z$, the function $h(.,Z)$ is not constant (which is often the case for not very detailed contextual variables), this does not allow for a reconstruction of $h$, and therefore for the determination of the relation between $X$ and $Y$.

In this paper is to refine a such procedure (or in the case where there are no contextual variables and therefore where no matching is possible, our aim is to link as well as we can $X$ and $Y$), and is thus to estimate, for any $Z$ in the set of contextual variables, the relation between $X$ and $Y$ given $Z$, i.e.~the function
$$ h_Z(.) := h(.,Z),$$
given the data $\mathbf X^Z \quad \mathrm{and} \quad \mathbf Y^Z$.

\noindent
\textbf{Revisited objective: Infer the function $h_Z$ from the data $\mathbf X^Z \quad \mathrm{and}\quad \mathbf Y^Z$.}

\paragraph{Remark on the model} The post-matching model is $y=h_Z(x^Z)+\epsilon$ - so that if for a given context $Z$ we observed a dataset of the form $(X^Z,Y^Z)$ (therefore with with cross-information), $X^Z$ and $Y^Z$ would not be independent - and one could use standard techniques as e.g.~regression. But in our setting we observe the data $X^Z$ and $Y^Z$ from different, independent sources - making de facto $X^Z$ and $Y^Z$ independent (and fully independent, not just knowing the order statistics). Regression techniques are not applicable there since we do not have the information of which $x^Z$ in dataset $X^Z$ corresponds to which $y^Z$ in dataset $Y^Z$.





\vspace{-0.3cm}
\section{Methods and results}\label{s:results}

\vspace{-0.3cm}

We now present our main procedure and results.

\vspace{-0.3cm}
\subsection{Main procedure and results}\label{ss:mr}

\vspace{-0.3cm}

We assume that we dispose of a matching procedure based on the contextual variables $Z$ as described in Subsection~\ref{ss:match}. We restrict to the case of discrete contextual variables for the theoretical results, and we consider exact matching according to these variables (the subsets $\mathbf X^Z \quad \mathrm{and} \quad \mathbf Y^Z$ correspond to exactly the same $Z$) but this can be easily generalised in practice. If no contextual variables $Z$ are available ($d=0$), then we use as convention in the rest of this section
$$\mathbf X^Z := \mathbf X \quad \mathrm{and} \quad \mathbf Y^Z := \mathbf Y.$$

Let $Z$ be a given value of the contextual variables such that $(\mathbf X^Z , \mathbf Y^Z)$ are non-empty and let $n_Z$ be the number of data in the smallest of these two sets. Let $\hat F_Z$ be the empirical distribution estimator of $F_Z$ defined over the samples $\mathbf X^Z$. We assume that there is a \textit{deconvolution estimator} $\hat F_{h,Z}$ that estimates the distribution $F_{h,Z}$, based on $\mathbf Y^Z$ (separating its density $g_Z = f_{h,Z} * \xi_Z$ form the noise $\xi_Z$), and satisfies the following assumption.  
\begin{assumption}[Deconvolution estimator available for the explained residuals]\label{ass:estFh}
Let $\delta>0$. Let $x \in \mathbb R$. There exists an estimator $\hat F_{h,Z}$ of $F_{h,Z}$ computed using the $\mathbf Y^Z$ and the knowledge of $\xi$ and that is such that with probability larger than $1-\delta$
\begin{align}\label{eq:estdec}
|F_{h,Z}(x) - \hat F_{h,Z}(x)| \leq \psi(\delta,n_z):=\psi.
\end{align}
\end{assumption}
The existence of a such \textit{deconvolution estimator}, satisfying Assumption~\ref{ass:estFh} is standard under some regularity conditions. A discussion on the existence of it is provided in the Subsection~\ref{app:decon}.



Theorems~\ref{th:cocobleu} and~\ref{th:cocobleu2} below give some properties on the efficiency of the estimate $\hat h_Z$ of $h_Z$ defined as
$$\hat h_Z = \hat F_{h,Z}^{-1}\circ \hat F_Z,$$
where $\hat F_{h,Z}^{-1}$ is the pseudo inverse of $\hat F_{h,Z}$. The entire procedure for computing $\hat h_Z$, which we call MatchMerge, for constructing this estimator of $h$ is summarized in Algorithm~\ref{algo1}.

\begin{algorithm}[t]
\caption{The procedure MatchMerge}
 \label{algo1}
 \begin{algorithmic}
 \STATE {\bfseries Input:}
 \STATE \quad $(\mathbf X, \mathbf Z(1)) \quad \mathrm{and}\quad (\mathbf Y, \mathbf Z(2))$
 \STATE \quad A matching method with respect to $Z$
 \STATE \quad A deconvolution method from the noise $\xi_z$
\smallskip
\STATE {\bfseries Main procedure :}
\STATE Apply the matching method to the data and obtain for all $Z$ $\mathbf X^Z \quad \mathrm{and}\quad \mathbf Y^Z$
\FOR{$Z$ s.t.~$n_Z>1$}
\STATE Compute the empirical estimator $\hat F_Z$ of $F_Z$ on $\mathbf X^Z$
\STATE Compute the deconvolution estimator $\hat F_{h,Z}$ of $F_{h,Z}$ using $\mathbf Y^Z$ and $\xi_Z$
\STATE Set $\hat h_Z = \hat F_{h,Z}^{-1}\circ \hat F_Z.$
\ENDFOR
\smallskip
\STATE {\bfseries Output :}
\STATE \quad Return $\hat h (.,.) = \hat h_.(.)$

 \end{algorithmic}
\end{algorithm}


The following theorem provides a first theoretical guarantees for $\hat h_Z(.)$, as well as a confidence statement in any point $u$.
\begin{theorem}\label{th:cocobleu}
Let $u\in \mathbb R$. Let Assumption~\ref{ass:estFh} be satisfied for $\delta>0$. We have with probability $1-2\delta$
$$h_Z \circ F_Z^{-1}(F_Z(u)-\psi - \phi) \leq \hat h_Z(u) \leq h_Z \circ F_Z^{-1}( F_Z(u) + \psi +\phi).$$
\end{theorem}
The proof of this result is in the Appendix, Subsection~\ref{pr:cocobleu}. This theorem provides a bound on the accuracy of the estimate of $\hat h_Z(u)$ if a rather mild condition (Assumption~\ref{ass:estFh} is verified). It is important to note here that apart from the assumption that $h_Z$ is a monotone function (without loss of generality, let us say that $h_Z$ is non-decreasing), no additional assumptions are made on $h_Z$. And even though the datasets $\mathbf X^Z$, $\mathbf Y^Z$ are independent, it is possible to recover the link function $h$ in this non-parametric model. The bound in this theorem is not explicit, it depends on $h_Z, F_Z, F_{h,Z}$. The next theorem gives an explicit bound (depending on $\phi, \psi$), provided that an additional assumption is made on $F_Z,h_Z$.

\begin{assumption}[H\"older assumption]\label{ass:Holder}
Let $(\alpha,L)>0$. A function $G$ is H\"older continuous on $\mathcal V$ if for any $(x,y)\in \mathcal V^2$, we have
$$|G(x)-G(y)| \leq L|x-y|^{\alpha}.$$
\end{assumption}
The H\"older assumption is mild if $\alpha$ is small and $L$ is large. In particular, functions that are differentiable on a compact (or Lipschitz) verify it with $L = \sup |G'|$ and $\alpha = 1$, but it is more general than this (functions of the form $x^{\alpha}$ verify it in $0$ for any $\alpha>0$ with parameters $\alpha$ and $L=1$). 

This theorem provides a more specific theoretical guarantee in the case where the distributions are H\"older smooth.
\begin{theorem}\label{th:cocobleu2}
Let $u\in \mathbb R$. Let Assumption~\ref{ass:estFh} be satisfied for $\delta>0$ and assume that for any $Z$, $h_Z$ is monotone (without loss of generality, non-decreasing). Let $L,M,\alpha,\beta>0$. Assume that $h_Z$ is $(\alpha,L)-$H\"older on $[u-M(\psi +\phi)^{\beta}, u + M(\psi +\phi)^{\beta}]$, and that $F_Z^{-1}$ is $(\beta,M)-$H\"older on $[F_Z(u)-\psi - \phi, F_Z(u) + \psi + \phi]$ (where $F_Z^{-1}$ is the pseudo-inverse of $F_Z$) as in Assumption~\ref{ass:Holder}.

Then with probability larger than $1-2\delta$
$$|\hat h_Z(u) - h_Z(u)| \leq  LM^{\alpha}(\psi +\phi)^{\alpha\beta}.$$
\end{theorem}
The proof of this result is in the Appendix, Subsection~\ref{pr:cocobleu2}.

The bound in Theorem~\ref{th:cocobleu2} implies that it is possible to recover the link function between the data $X,Y,Z$, 
$$h(x,z) = h_z(x),$$
with a good precision. Our approach enables us to complement matching in this way. Again, if no contextual variables $Z$ are available ($d = 0$), our approach is still applicable by just skipping the matching step.

\paragraph{Remark on the noise assumption} In order for our theoretical results to holds, we need to know the distribution of the noise to perform the deconvolution. Let $f^*$ be the Fourier transform of the noise, and $f$ be the Fourier transform of a distribution. If $|1-1/f^*|\geq|1/f-1/f^*|$, i.e.~if the distance defined as $d(g,g') = |1/g-1/g'|$, between the Dirac Fourier transform and $f^*$ is larger than the distance between $f$ and $f^*$, then deconvolution with $f$ is better than quantile matching. The distance relation stated before is quite sensitive to parameters such as variance or range, so if $f^*$ is closer to $f$ in terms of these quantities than to a Dirac mass in $0$, it is likely that our method applied using $f$ in the deconvolution step will be more efficient than simple quantile matching.




\vspace{-0.3cm}
\subsection{Special case when $\tilde h$ is separable : regression approach}\label{sec:compl2}
\vspace{-0.3cm}

The model we have been considering until now is quite general. In many cases, however, the effect of the contextual variables $Z$ can be separated from the effect of $X$.

The underlying model is then as follows. Let $d \in \mathbb N$ be the dimension of the control variables $Z$. 

We assume the following model for the explicative variables $X$:
$$X := h_1(Z) + \tilde X,$$
where $h_1$ belongs to some functional class $\mathcal H$, and where $\tilde X \sim f$ is independent of $Z$, where $f$ is a density on $\mathbb R$ (with corresponding distribution $F$). We now assume the following modelling for the explained variable $Y$:
\vspace{-0.3cm}
$$Y := h(\tilde X) + h_2(Z) + \epsilon,\vspace{-0.2cm}$$
where $\epsilon \sim \xi$ is independent of $(Z, \tilde X)$, where $\xi$ is a density on $\mathbb R$ of mean $0$, where $h_2$ belongs to some functional class $\mathcal H$, and where $h$ is an increasing function. Typically, $\epsilon$ is a white noise (e.g.~Gaussian). We assume that we know its distribution $\xi$. Let us write
\vspace{-0.3cm}
$$\tilde Y := Y - h_2(Z).$$
We write $f_h$ for the density of $h(\tilde X)$ (with corresponding distribution $F_h$), and $g$ for the density of $\tilde Y$ (with corresponding distribution $G$). Given the model, we have
\vspace{-0.3cm}
$$g = f_h * \xi,$$
i.e.~$g$ is the convolution of $f_h$ and $\xi$.

A classical model which is separable is the popular linear model. Real data are in general not exactly separable, but it is often a useful approximation to assume that they are - at least locally.

The objective here is to find $h$, knowing that it is increasing. The following methodology can be applied.
First, estimate $h_1$ and $h_2$ using respectively $(\mathbf X, \mathbf Z(1))$ and $(\mathbf Y, \mathbf Z(2))$. These two problems are standard functional regression problems, and usual methods can be used here (linear regression, Kernel regression, regression on a base, SVM, etc, depending on the set $\mathcal H$ containing $h_1,h_2$). Using these estimates, one can deduce estimates $\mathbf{\hat X}$ and $\mathbf{\hat Y}$ for $\mathbf{\tilde X}$ and $\mathbf{\tilde Y}$. One can the perform the method described in Subsection~\ref{ss:mr} on the data $\mathbf{\hat X}$ and $\mathbf{\hat Y}$ with null contextual variables. The method is described in Algorithm~\ref{algo2}.


\begin{algorithm}[t]
\caption{The procedure MatchMergeSep}
 \label{algo2}
 \begin{algorithmic}
 \STATE {\bfseries Input:}
 \STATE \quad $(\mathbf X, \mathbf Z(1)) \quad \mathrm{and}\quad (\mathbf Y, \mathbf Z(2))$
 \STATE \quad A regression method
 \STATE \quad A deconvolution method from the noise $\xi$
\smallskip
\STATE {\bfseries Main procedure :}
\STATE Compute the regression estimator $\hat h_1$ of $h_1$ using $(\mathbf X, \mathbf Z(1))$
\STATE Compute the regression estimator $\hat h_2$ of $h_2$ using $(\mathbf Y, \mathbf Z(2))$
\FOR{$Z$ s.t.~$n_Z>1$}
\STATE Set $\mathbf{\hat Y} := \mathbf{Y} - \hat h_2(\mathbf{Z}(1))$, and $\mathbf{\hat X} := \mathbf{X} - \hat h_1(\mathbf{Z}(1))$
\STATE Compute the empirical estimator $\hat F$  of $F$ on $\mathbf{\hat X}$
\STATE Compute the deconvolution estimator $\hat F_{h}$ of $F_{h}$ using $\mathbf{\hat Y}$ and $\xi$
\STATE Set $\hat h = \hat F_{h}^{-1}\circ \hat F.$
\ENDFOR
\smallskip
\STATE {\bfseries Output :}
\STATE \quad  Return $\hat h(.)$

 \end{algorithmic}
\end{algorithm}

\vspace{-0.3cm}
\section{Experiments}\label{s:expes}
\vspace{-0.3cm}

In this section, we apply our methods on two real datasets : the first dataset contains macro-economic data from the world bank on development in $45$ African countries, and the second contains micro-economic data from the UK land registry on house transaction prices in London. In order to assess the performance of our method, we have deliberately chosen datasets which contain simultaneously the two variables we want to merge and therefore their relationship - but we split in these two applications these datasets when we apply our method and consider the first variable on the first part and the second variable on the second part, so that the two variables are independent and so that we are in real condition for applying our method. The datasets we use are publicly available and more informations on them can be found in the description of the experiments. We also performed simulations and applied our method on synthetic data, this is to be found in the Appendix~\ref{s:expes2} due to space constraint.

\vspace{-0.3cm}
\subsection{Experiment on real data I: Urbanisation and life expectancy}\label{s:expes1}
\vspace{-0.3cm}

In the first example we look at the impact of urbanisation on life expectancy in $45$ African countries, in $2006$. We use world bank macro-economic data (the Africa Development Indicators (ADI) dataset). We consider the two variables $\mathbf Y=$``average life expectancy" and $\mathbf X=$``urbanisation percentage". The life expectancy is clearly an increasing function of urbanisation since a more urbanised country implies that more infrastructure (electricity, hospitals, etc) is available, and that accessing these public services is much easier.

In this dataset, the cross-information is available (the data-set provides, for each country, the average life expectancy and the urbanisation percentage). This cross information will be used for performance assessment only, and of course, not by our method. 
We do not provide our algorithm with the dependence structure $(\mathbf X,\mathbf Y)$, but we provide it with independent sub-samples of size $30$ of $\mathbf X$ and $\mathbf Y$. We plot the estimate $\hat h$ obtained with our method using different deconvolution distributions in the deconvolution process (the ``true" noise distribution $\xi$ is not available), as well as the points $(X_i,Y_i)$ for illustrating how the estimator of the link function captures the dependence structure from the monotonicity constraint. The results are plotted in Figure~\ref{fig:exp4}, each curve corresponds to an estimate $\hat h$ for which the deconvolution is made using a different deconvolution distribution. We provide results for normal deconvolution distribution with various variances (left plot), and uniform deconvolution distribution with various ranges (right plot). 
As expected, the smaller the variance of the deconvolution distribution, the closer the estimator $\hat h$ is to the points from which it is constructed. However, as in standard regression, this comes together with the problem of over-fitting, as illustrated in Table~\ref{ta1} where the risk (the square root of the MSE) is evaluated on an independent sample of 20 countries. This evaluation shows clearly that taking a deconvolution distribution of small variance does not provide an estimator that has good generalization properties, and it is therefore better to take this into account instead of doing a simple quantile matching.

From the curves in Figure~\ref{fig:exp4}, it seems that urbanisation has a multiplier effect on the life expectancy: indeed, the function $h$ is convex until a threshold. However, there is clearly endogeneity in this model. In particular, the richer the country, the more urbanised it is (in Africa). And the richer the country, the higher the life expectancy. If we want to measure the true effect of urbanisation on life expectancy, we should get rid of this side effect. We should thus control for this effect using a control variate $Z=$``GDP per head", as explained in Subsection~\ref{sec:compl2}. We assume here that there is a linear underlying model, and that we have
\vspace{-0.2cm}
$$X_i = \tilde X_i + \alpha_1  + \beta_1 Z_i \quad \mathrm{and} \quad Y_i = \tilde Y_i + \alpha_2  + \beta_2 Z_i.\vspace{-0.2cm}$$
We estimate the $\tilde X_i$ and the $\tilde Y_i$ using the control $Z$ and doing a linear regression (as in Algorithm~\ref{algo2}). As explained before, urbanisation, even after the control should have a positive effect on the controlled life expectancy, because urbanisation enlarges the access to important public services. The results in this controlled setting are plotted in Figure~\ref{fig:exp5}. The points are the estimates of the controlled $(\tilde X_i,\tilde Y_i)$, and each curve corresponds to a deconvolution with a different noise deconvolution. We observe here that the curve $\hat h$ is now concave - controlling by the GDP per head has cancelled the multiplier effect, although the impact of urbanisation on life expectancy is still positive. It means that some urbanisation has a very positive effect on life expectancy (because it implies better access to vital infrastructure), but that this positive effect is sub linear (the multiplier effect, coming from the fact that the GDP is positively correlated with urbanisation and life expectancy, is suppressed). The MSE (again, evaluated on an independent sample) is displayed in Table~\ref{ta2}
\vspace{-2cm}

\begin{center}
\begin{figure*}[!htb]
\begin{minipage}{8.4cm}
\includegraphics[width=8.4cm]{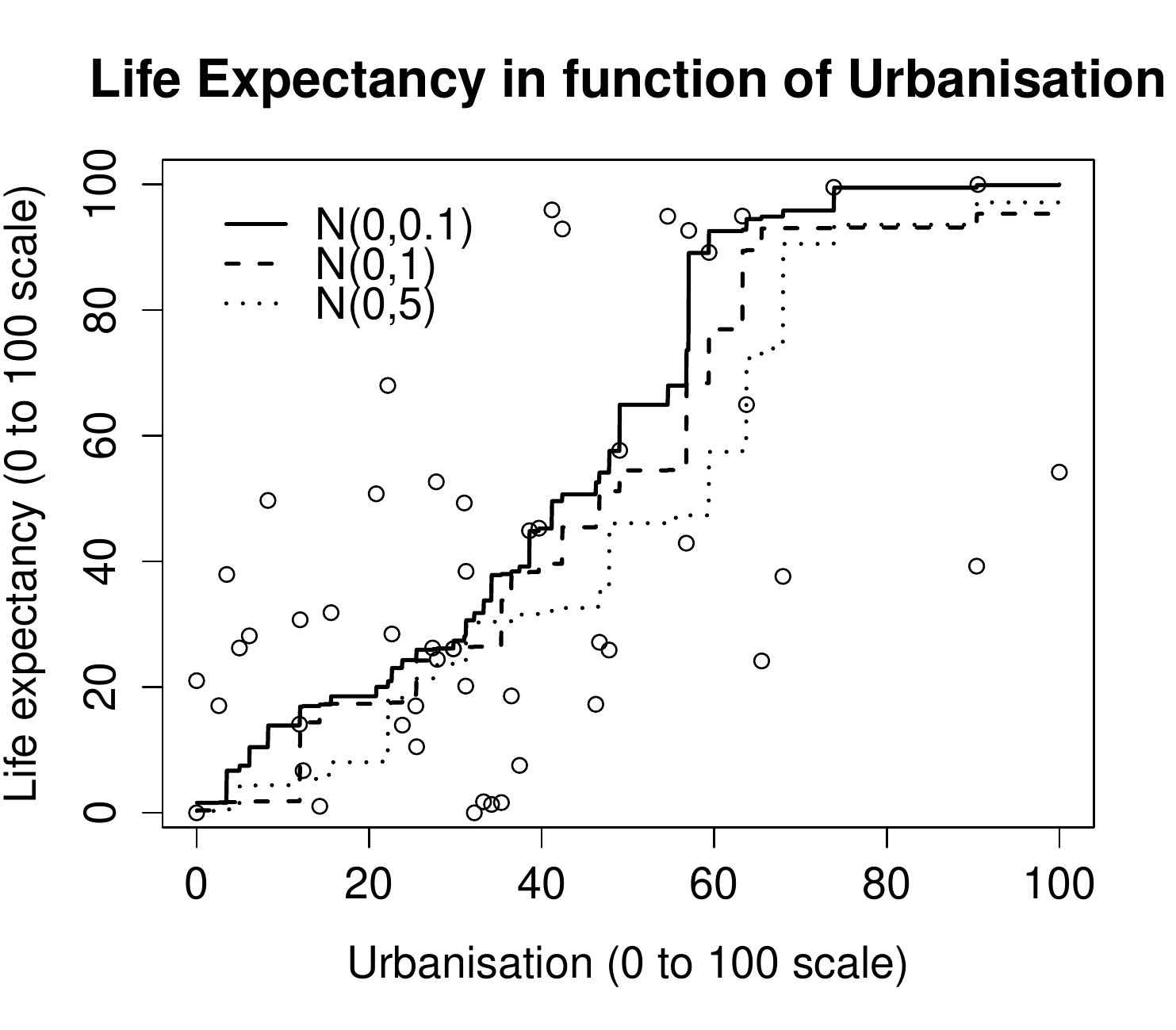}
\end{minipage}
\begin{minipage}{8.4cm}
\includegraphics[width=8.4cm]{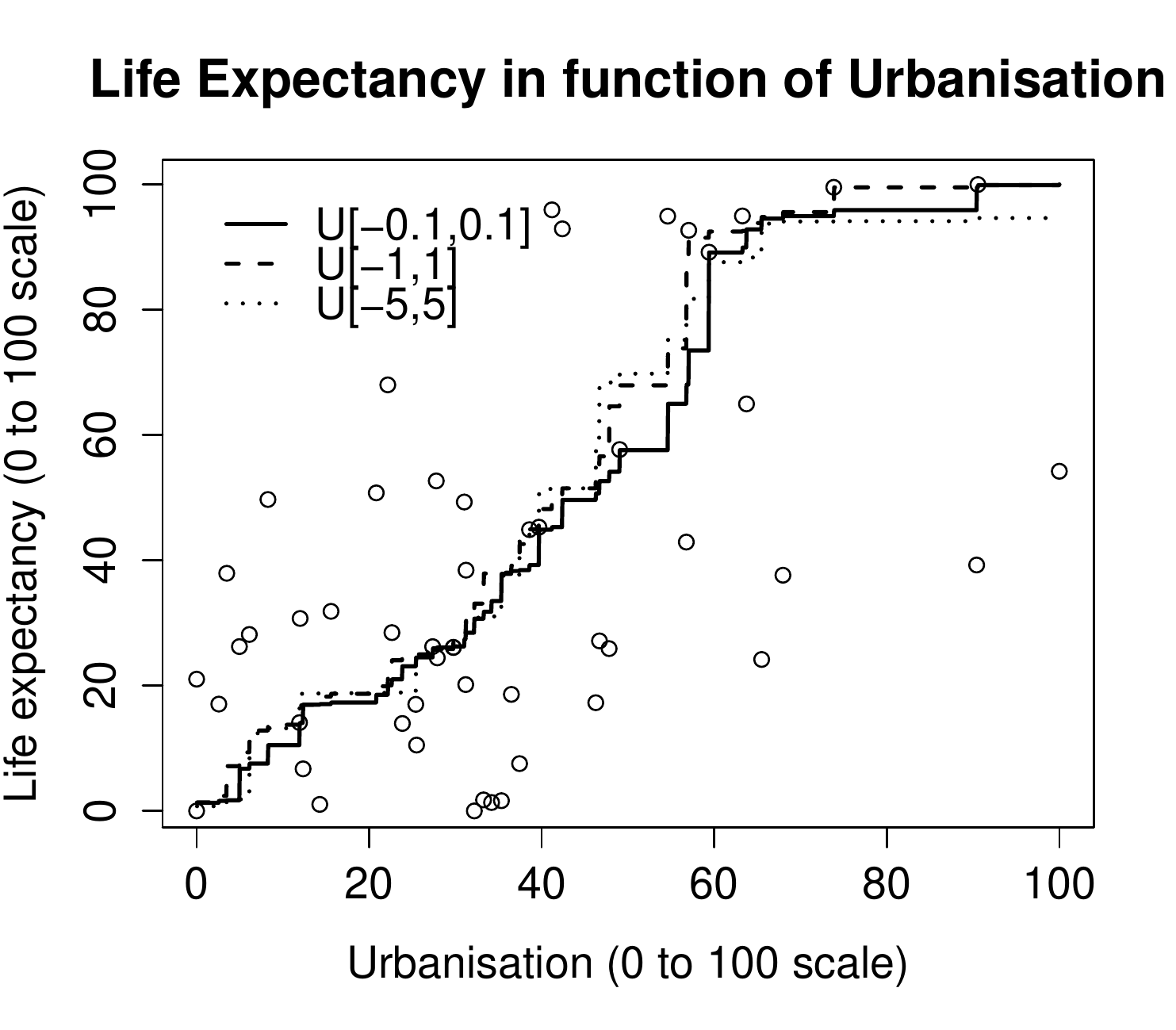}
\end{minipage}
\caption{Life expectancy in function of urbanisation percentage (both rescaled on $[0,100]$) in $45$ African countries. Each point corresponds to a country. The lines correspond to $\hat h$ estimated using different deconvolution distributions i.e. $\mathcal N(0,0.1)$, $\mathcal N(0,1)$ and $U([-5,5])$.} \label{fig:exp4}
\end{figure*}
\end{center}

\begin{table*} \centering 
\caption{Risk of the estimates obtained in Figure~\ref{fig:exp4}.} \label{ta1}
\begin{tabular}{@{}l rrr c rrr@{}} 
\hline
&\multicolumn{3}{c}{Normal distribution} && \multicolumn{3}{c}{Uniform distribution}\\ 
\hline
Distribution & $\mathcal N(0,0.1)$ & $\mathcal N(0,1)$ & $\mathcal N(0,5)$ &  \phantom{abc} & $U([-0.1,0.1])$ & $U([-1,1])$  & $U([-5,5])$ \\ 
\hline
MSE & 40.4 & 37.5 &  38.2 & \phantom{abc} & 39.4 & 31.2  & 33.4 \\

\hline
\end{tabular} 
\end{table*}

\begin{center}
\begin{figure*}[!htb]
\begin{minipage}{8.4cm}
\includegraphics[width=8.4cm]{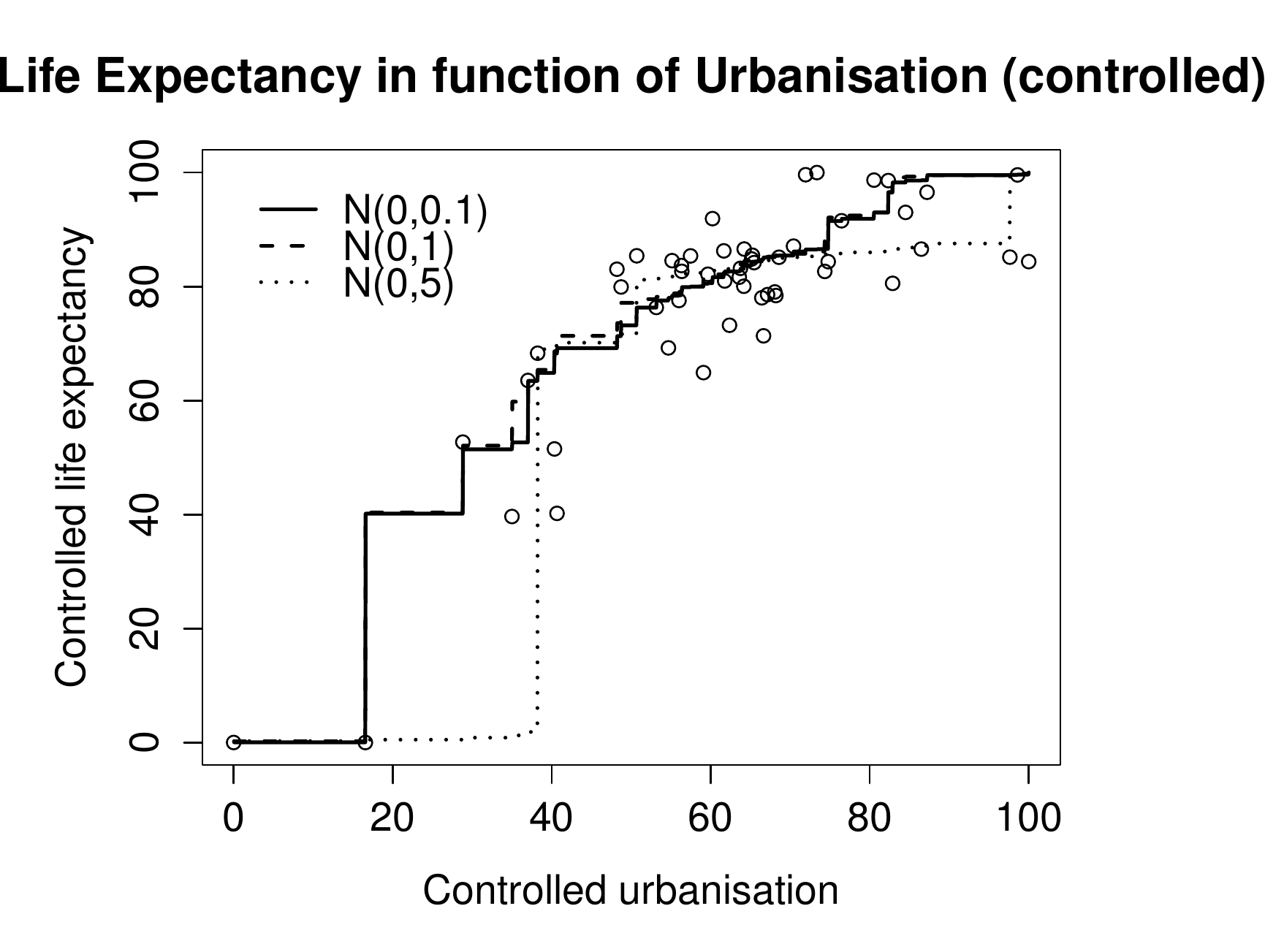}
\end{minipage}
\begin{minipage}{8.4cm}
\includegraphics[width=8.4cm]{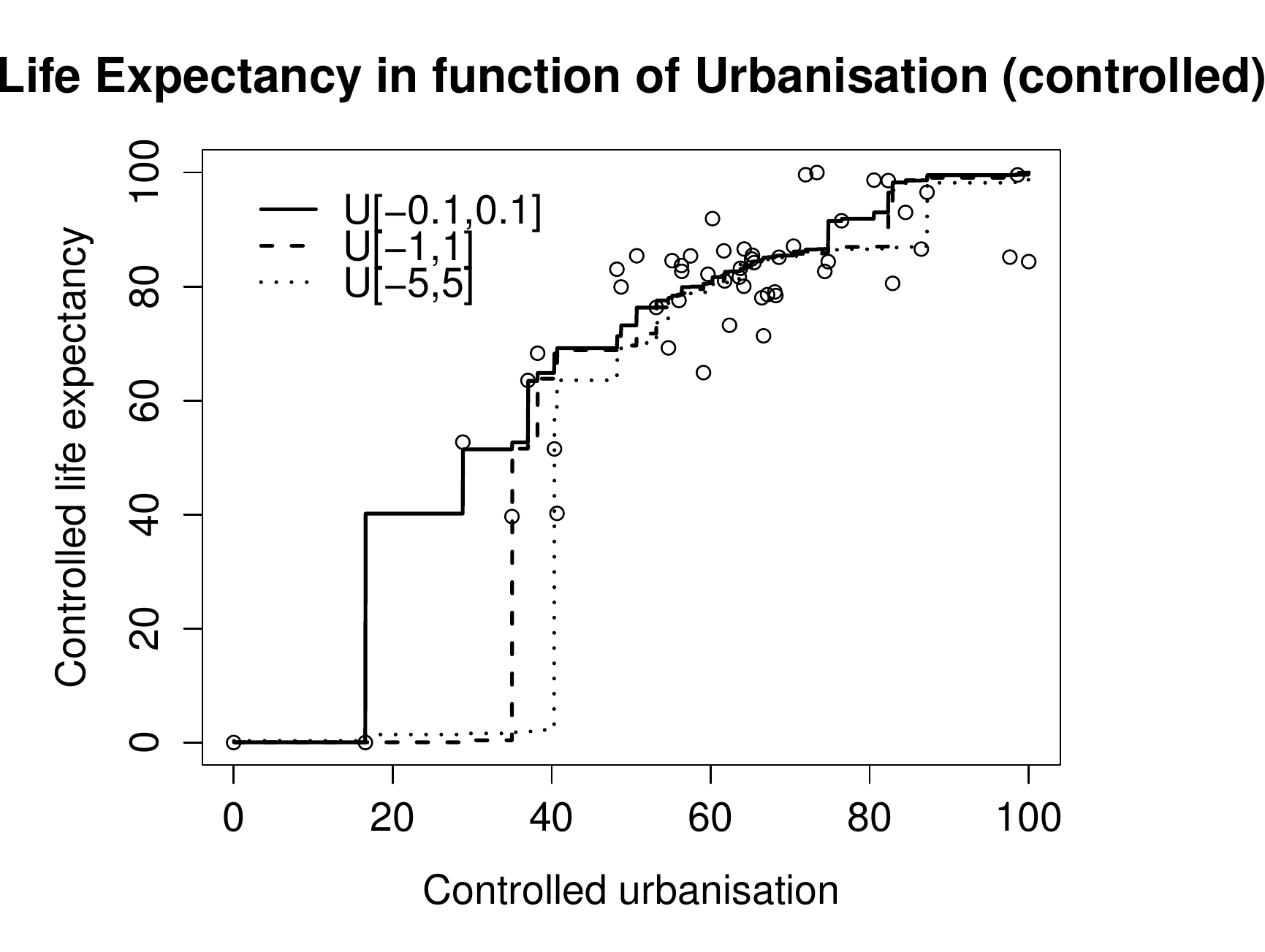}
\end{minipage}
\caption{Controlled life expectancy in function of controlled urbanisation percentage (both rescaled on $[0,100]$) in $45$ African countries. Each point corresponds to a country. The lines correspond to $\hat h$ estimated using different deconvolution distributions i.e. $\mathcal N(0,0.1)$, $\mathcal N(0,1)$ and $U([-0.5,0.5])$.} \label{fig:exp5}
\end{figure*}
\end{center}
\begin{table*}[!htb] \centering 
\caption{Risk of the estimates obtained in Figure~\ref{fig:exp5}.} \label{ta2}
\begin{tabular}{@{}l rrr c rrr@{}} 
\hline
&\multicolumn{3}{c}{Normal distribution} && \multicolumn{3}{c}{Uniform distribution}\\ 
\hline
Distribution & $\mathcal N(0,0.1)$ & $\mathcal N(0,1)$ & $\mathcal N(0,5)$ &  \phantom{abc} & $U([-0.1,0.1])$ & $U([-1,1])$  & $U([-5,5])$ \\ 
\hline
MSE & 17.4 & 7.9 &  8.6 & \phantom{abc} & 15.5 & 7.3  & 8.2 \\

\hline
\end{tabular} 
\end{table*}

\vspace{-0.3cm}
\subsection{Experiment on real data II : Property prices and percentage of educated residents}\label{s:expes2}
\vspace{-0.3cm}

In the second example we look at the impact of the the neighborhood share of high skilled residents on local property prices. We use data from the 2011 UK census on the share of residents holding a university degree in electoral wards and prices of 2011 housing transactions available for download at the land registry website\footnote{http://www.landregistry.gov.uk/market-trend-data/public-data/transaction-data}. We consider the two variables $\mathbf Y=$``average price" and $\mathbf X=$``percentage of high skilled residents". House prices in an area are clearly increasing in the local concentration of well educated workers as workers holding university degrees are paid the highest wages and subsequently spend more on housing. We are interested in reconstructing the map of house prices in London, using the percentages of degrees (with geographical co-variate), and the house prices (but without using the geographical information) : we want to merge the percentages of degrees (plus geographical location) with the house prices. We believe that this example is interesting, because unlike the UK, most countries do not provide refined geographical data for house transaction, whereas geographical census data are usually available in developed countries. 

Figure \ref{fig:exp6} a) shows the distribution of average house prices in 2011 for 733 London wards. House prices are highest in the west of Inner London and in the south-west of Outer London. With our method we try to reconstruct the local price pattern on the basis of the degree variable without making use of the geographical information available for the house prices. We divide at random our sample in two datasets, one that will serve the purpose of constructing the estimate and the other the one of evaluating it through the MSE (300 for construction of the estimates and 433 for evaluating their performances). Figure \ref{fig:exp6} c) and d) show reconstructions of the original price map using respectively the distributions $U([-0.5,0.5])$, and $U([-2.5,2.5])$ for the deconvolution (the ``true" noise distribution of the noise). The choice of a uniform deconvolution distribution $\xi$ is reasonable in this case (since the variables are bounded and rescaled between $0$ and $100$), and we considered two ranges for $\xi$ (1 and 5). As a comparison method we also  reconstruct the map based on quantile matching, which is equivalent to using a Dirac mass in $0$ in the deconvolution step (Figure \ref{fig:exp6} b)). All three methods are able to generate a price pattern very similar to the original spatial distributing shown in a). The qualitative difference is that the larger the noise with which one deconvolves, the more contrasted the picture gets - removing the smoothing that is due to the noise - and d) is the most contrasted picture. 

Figure \ref{fig:exp6bis} shows the estimate $\hat h$ obtained with our method using the different deconvolution distributions, plotted with the second part of the sample for evaluation. We evaluate the performance under the different deconvolution distributions by calculating the mean squared error (MSE), using the distance of each point to the function $\hat h$. Quantile matching results in a MSE of \textbf{8.62}, whereas the uniform distributions yield an MSE of \textbf{8.54} for $\xi=U([-0.5,0.5])$, respectively \textbf{8.23} for $\xi=U([-2.5,2.5])$. Our method outperforms a simple quantile match, when using a reasonable deconvolution distribution (here uniform, since the noise is bounded). 

\paragraph{Remark regarding the deconvolution noise in the empirical results} Empirically, if the noise distribution is unknown, we found out that in all studied cases, deconvolving with some "reasonable" distribution - even if it is not the correct one - outperforms quantile matching. The "reasonable" distribution does not need to be the true noise distribution, which is unknown, but it can be a distribution deduced from some prior knowledge on the noise. This knowledge does not need to be precise. In the examples of Subsection~\ref{s:expes1}, deconvolving with a distribution of very small variance is giving bad results with respect to the MSE while bigger noises are performing significantly better, see Table 1 and 2 (improvement ranges between 30\% to more than 50\% in Table 2). We did not show results for pure quantile matching because they are slightly worse than deconvolution with $\mathcal N(0,0.1)$. Also, deconvolution with the uniform distribution or the Gaussian distribution of variance ranging from $1$ to $2\times 5^2$ provides comparable results - although these distributions are significantly different and largely span the set of "reasonable" distributions. Similar results can be seen in the MSE in Figure 4 for the second example and in the Appendix on synthetic data. This highlights the fact that one does not need to have a precise knowledge of the noise distribution to gain something significant with respect to quantile matching.

\vspace{-0.3cm}

\section*{Conclusion} \vspace{-0.3cm}We developed in this paper a new method for merging variables. It can be used as a complement to matching - or also on its own when no contextual variables are available. It is easy to implement and provides good results, both in theory and in practice, provided that the dependence function $h$ is monotone.

\begin{figure*}[!p]
\begin{minipage}{8.4cm}
a)\\
\includegraphics[width=7.4cm, trim=0 0 3cm 0]{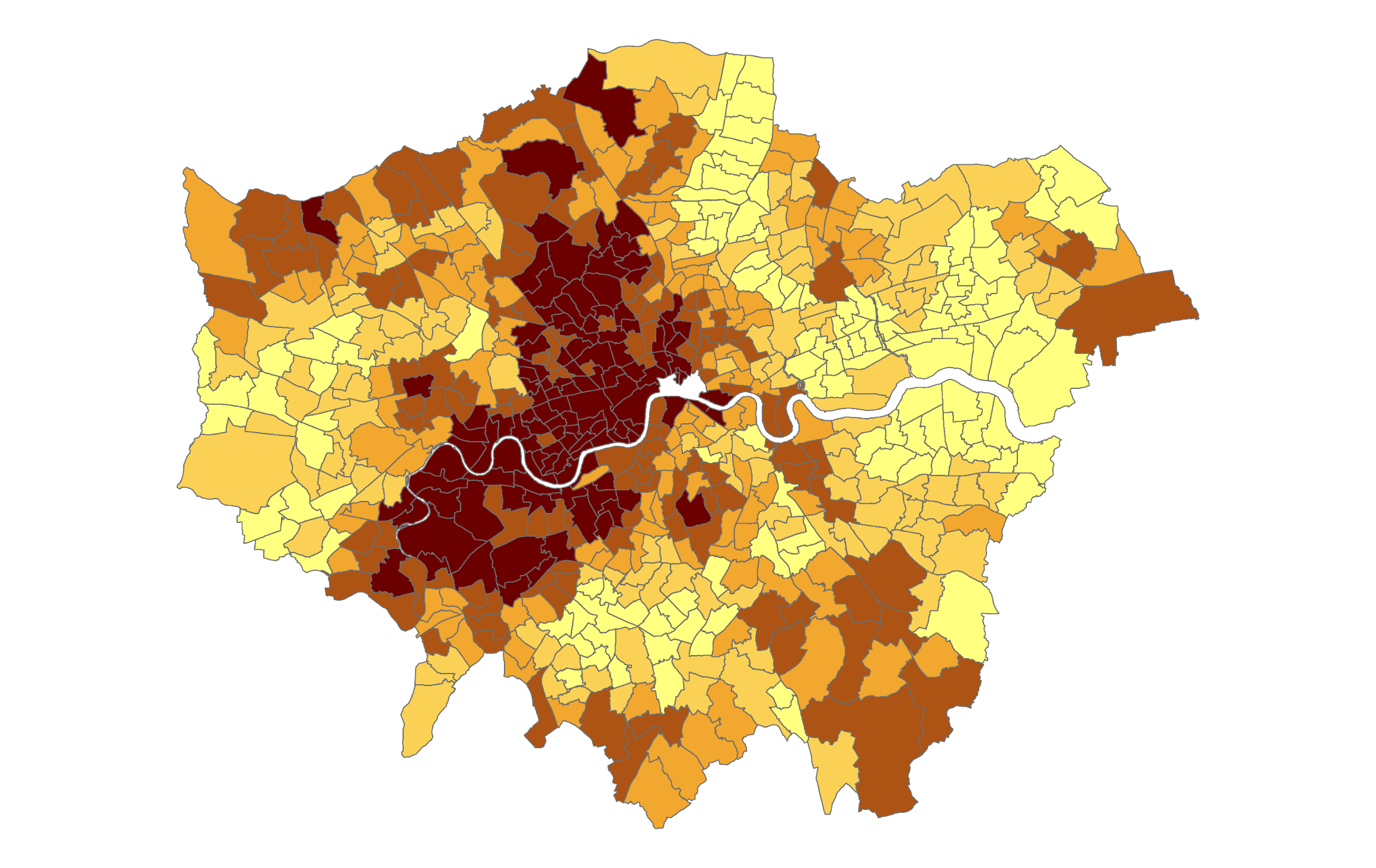}
\end{minipage}
\begin{minipage}{8.4cm}
b)\\
\includegraphics[width=8.4cm, trim=0 0 3cm 0]{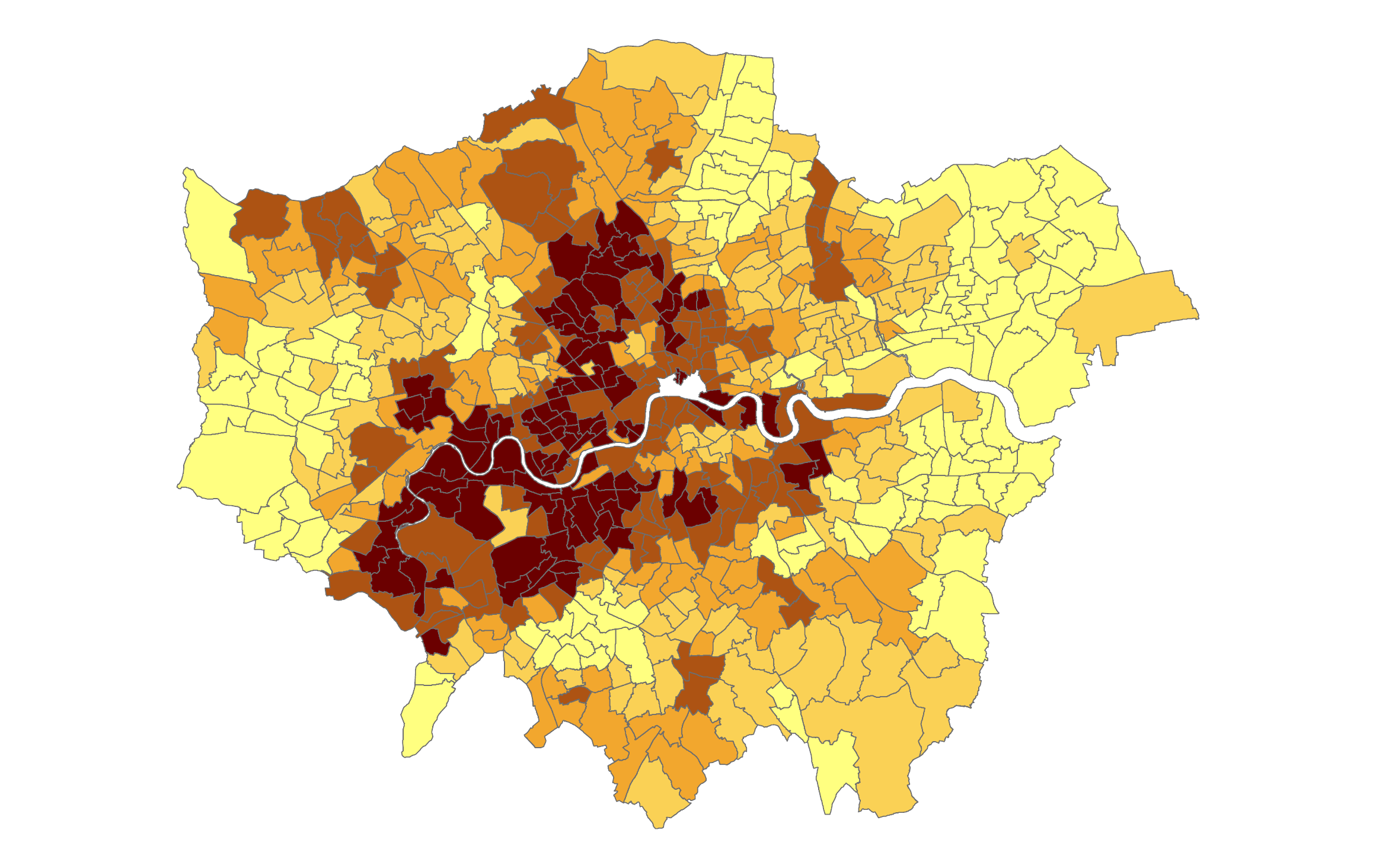}
\end{minipage}
\begin{minipage}{8.4cm}
c)\\
\includegraphics[width=8.4cm, trim=0 0 3cm 0]{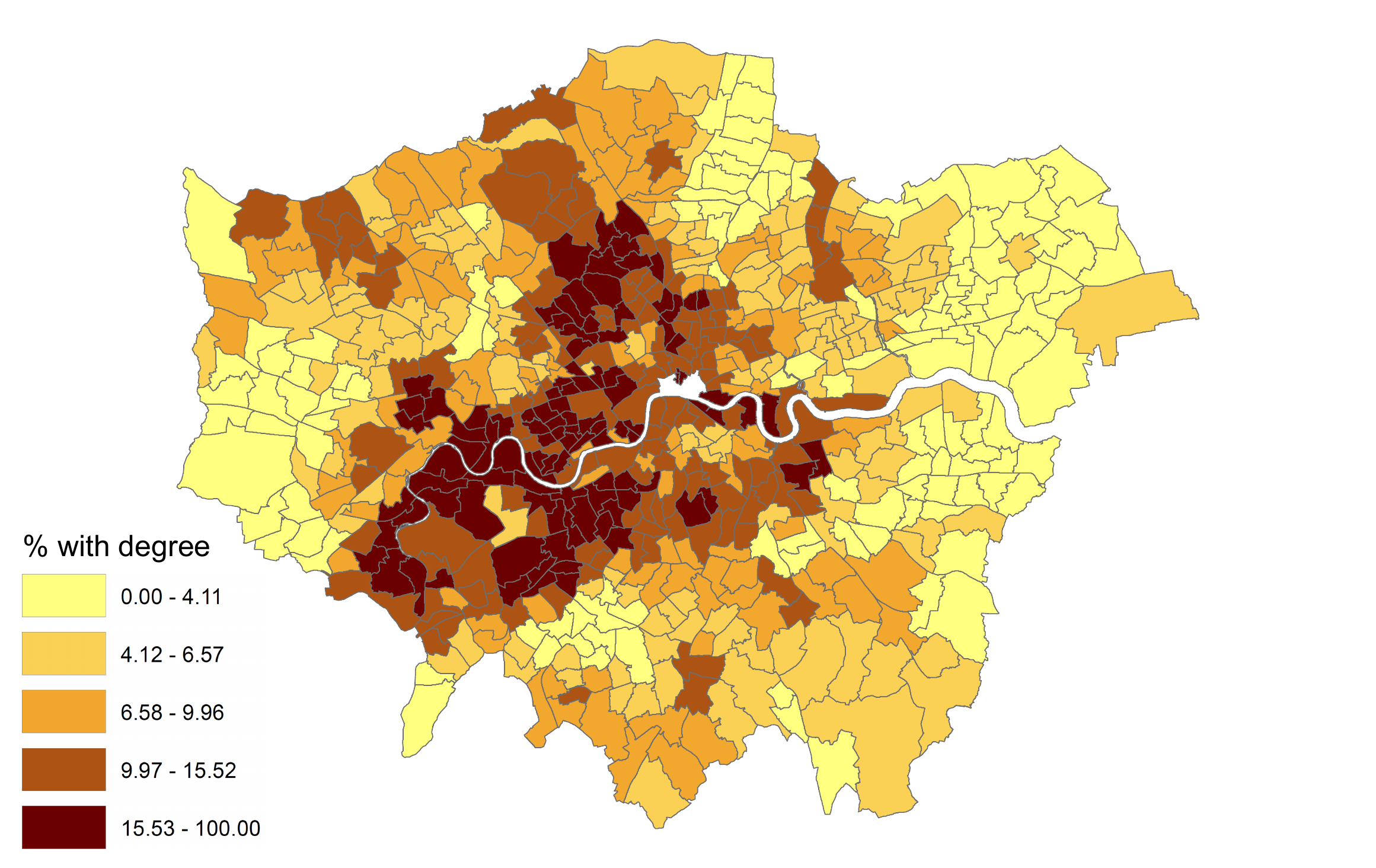}
\end{minipage}
\begin{minipage}{8.4cm}
d)\\
\includegraphics[width=8.4cm, trim=0 0 3cm 0]{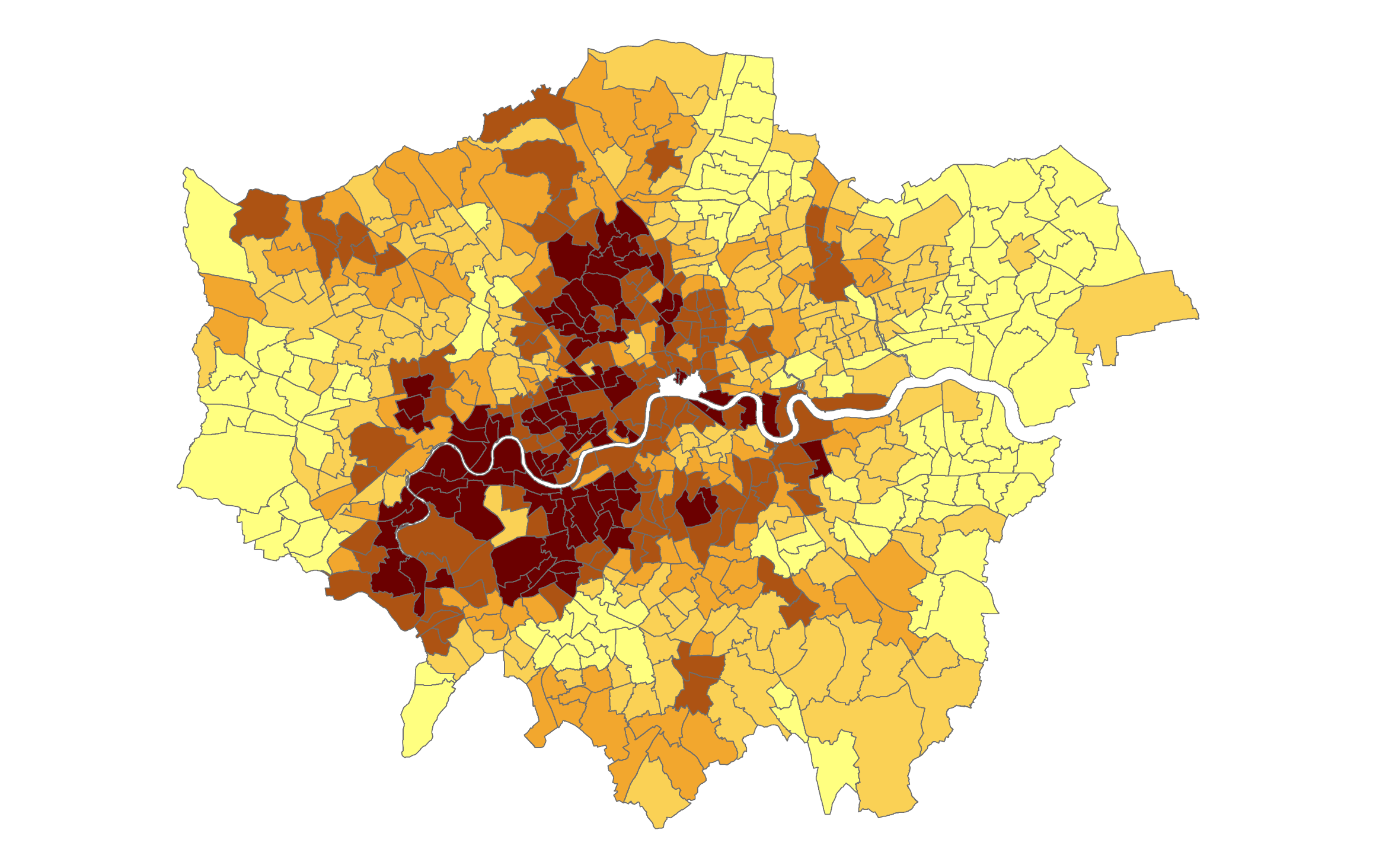}
\end{minipage}
\caption{Reconstruction of the house price map using the percentage of residents holding a degree in each ward. Maps from Up to Down and Left to Right: (UL) True map of average house prices in 733 London wards, (UR) Reconstruction of the map based on quantile matching $\xi=0$, (DL) Reconstruction of the map based on our method with $\xi=U([-0.5,0.5])$ and (DR) Reconstruction of the map based on our method with $\xi=U([-2.5,2.5])$} \label{fig:exp6}
\begin{minipage}{5.5cm}
\includegraphics[width=6cm]{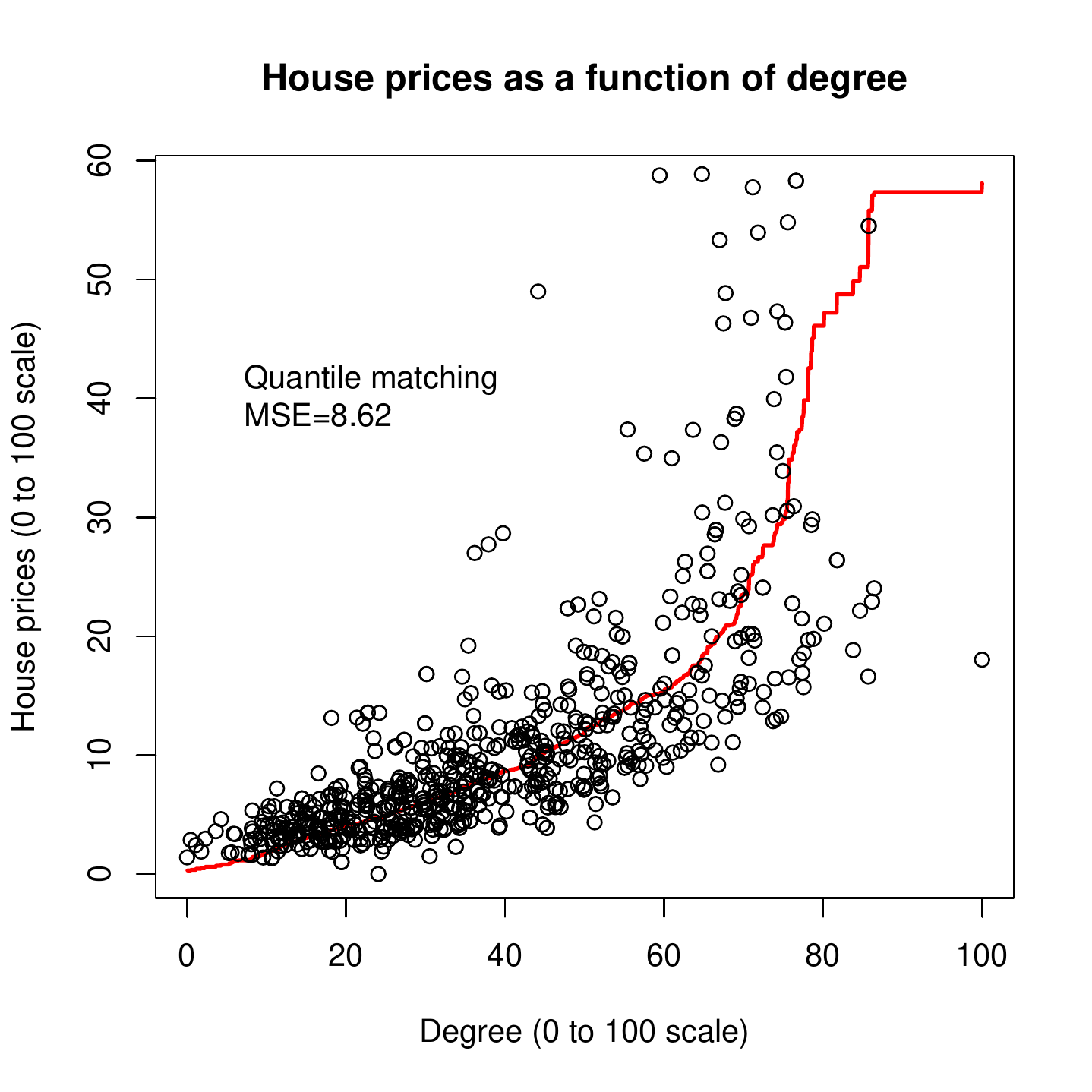}
\end{minipage}
\begin{minipage}{5.5cm}
\includegraphics[width=6cm]{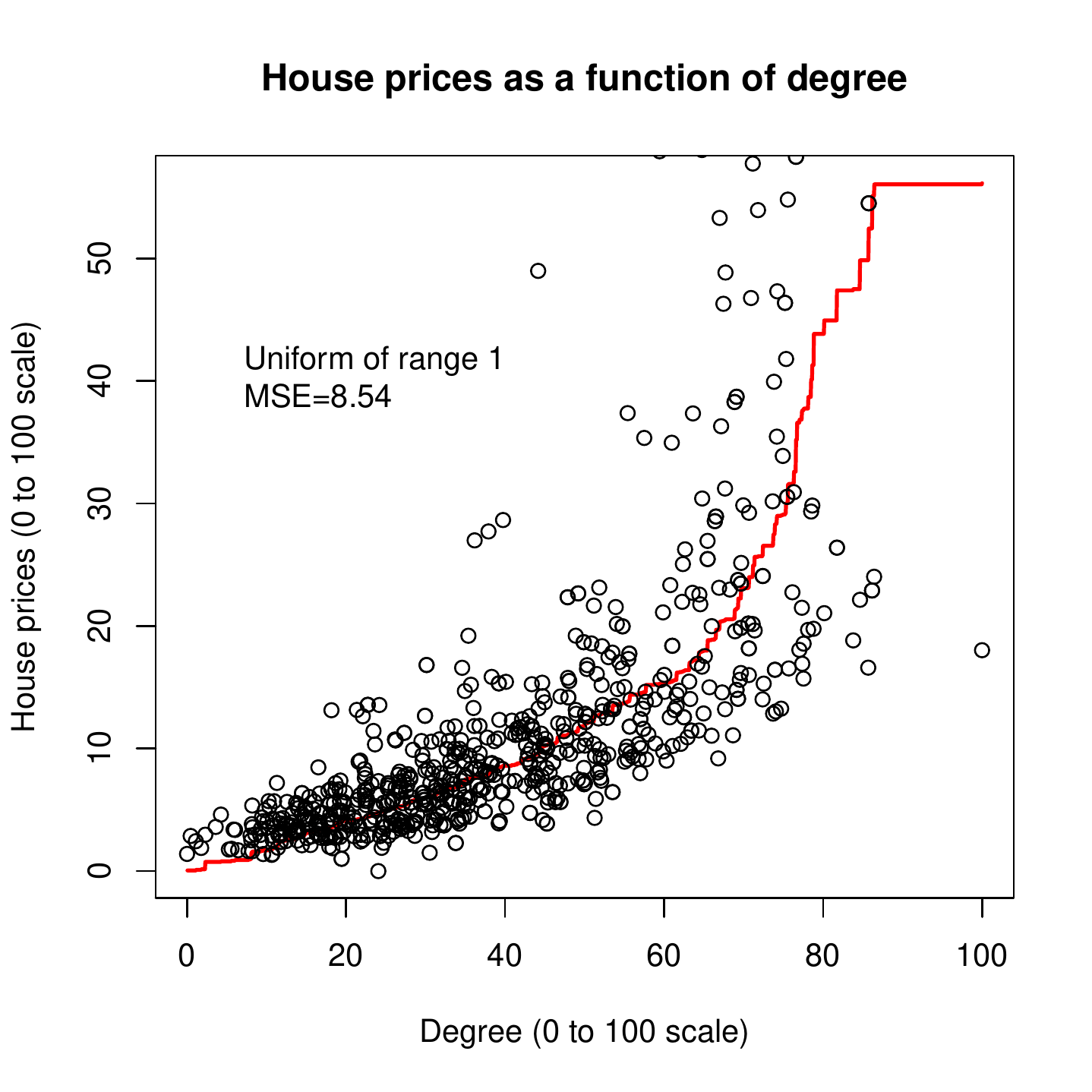}
\end{minipage}
\begin{minipage}{5.5cm}
\includegraphics[width=6cm]{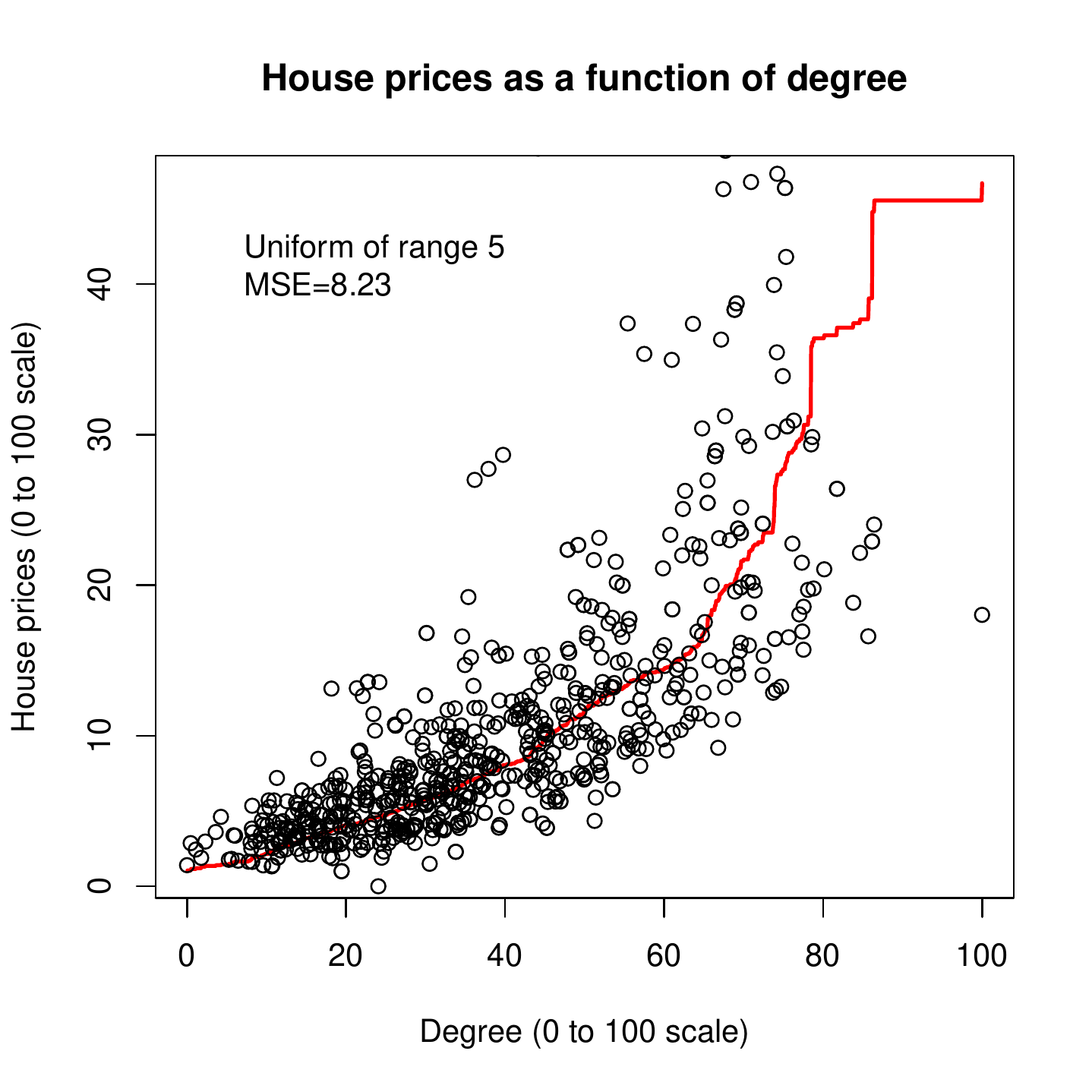}
\end{minipage}
\caption{House prices in a geographical unit in function of degree percentage in the same geographical unit (both rescaled on $[0,100]$) in $733$ London wards. Each point corresponds to a ward. The red lines correspond to $\hat h$ estimated using different distributions $\xi$. From left to right: quantile matching ($\xi = 0$), $U([-0.5,0.5])$ and $U([-2.5,2.5])$. The MSE (mean squared error) is the mean squared error between the line and the points.} \label{fig:exp6bis}
\end{figure*}

\paragraph{Acknowledgements} Part of this work was produced when AC was in the StatsLab in the University of Cambridge. AC's work is supported since 2015 by the DFG's Emmy Noether grant MuSyAD (CA 1488/1-1).

\newpage
{

\bibliography{allocation,bib}
\bibliographystyle{plain}

}


\newpage

\appendix

\section{Proofs and technical results}


\subsection{Existence of a deconvolution estimator}\label{app:decon}

There is a huge literature on deconvolution, and many estimators are available, see e.g.~\cite{dattner2011deconvolution,chan1998total, fan1991optimal, cordy1997deconvolution, carroll1988optimal, bell1995information, chan1998total,moulines1997maximum, starck2003wavelets}. A simple example of a such estimator is as follows. Let $\mathcal F(.)$ be the Fourier transform, and $\mathcal F^{-1}(.)$ be the inverse Fourier transform, and let $\hat P_{\mathbf Y^Z}$ be the empirical distribution of $\mathbf Y^Z$. An estimator of the density $f_{h,Z}$ can be defined as follows
$$\hat f_{h,Z} = \mathcal F^{-1}\Bigg(\frac{\mathcal F(\hat P_{\mathbf Y^Z})\Big|_{\tau}}{\mathcal F(\xi_Z)}\Bigg),$$
where $\mathcal F(\hat P_{\mathbf Y^Z})\Big|_{\tau}$ is the truncated Fourier transform of $\mathbb P_{\hat Y^Z}$ where any coefficients smaller than $\tau$ in absolute value are set to $0$. The estimator presented in Equation~\eqref{eq:estdec} is a very simple estimate, and is not always optimal. Many other estimators have been proposed in the previously quoted papers, are more effective, and should be preferred to this one in practice.

The plug-in estimate of the distribution function associated to $\hat f_{h,Z}$, i.e.~$\hat F_{h,Z}(x) = \int_{-\infty}^x \hat f_{h,Z}(u)du$ of $F_{h,Z}$ will converge to $F_{h,Z}$ at a rate $\psi$ which depends on the regularity $s$ of $f_{h,Z}$ (the more regular $f_{h,Z}$, the faster the rate), and of the decay of the Fourier spectrum of $\xi_Z$ (the heavier the tails of $\mathcal F(\xi_Z)$, the faster the rate), see e.g.~\cite{dattner2011deconvolution}. For instance, if the function $f_{h,Z}$ is in a H\"older ball of smoothness $s$ (see Assumption~\ref{ass:Holder} below), and if the noise $\xi_Z$ is ``ordinary smooth" (polynomial decay of $\mathcal F(\xi_Z)$), then $\psi(n_Z,\delta) = C (n_Z/\log(1/\delta))^{-\epsilon}$, where $C,\epsilon>0$ depend on $s$ and on the polynomial decay rate of $\mathcal F(\xi_Z)$. In e.g.~Gaussian noise, the rate is worse because the Fourier transform of a Gaussian decays very quickly, and the rate is $\psi(n_Z,\delta) = C (\log(n_Z/\delta))^{-\epsilon}$, where $\epsilon$ depends on $s$.

\subsection{Proof of Theorem~\ref{th:cocobleu}}\label{pr:cocobleu}

In this proof, in order to avoid notational heaviness, we write for convenience $\mathbf X$ instead of $\mathbf X^Z$ and $\mathbf Y$ instead of $\mathbf Y^Z$. We also write $F$ for $F_Z$ and $F_h$ for $F_{h,Z}$ (and the same simplification for the estimators).

We first state the two following lemma.

\begin{lemma}[Estimation of the distribution $F$]\label{lem:estF}
Let $\delta>0$. Let $x \in \mathbb R$. There exists an estimator $\hat F$ of $F$ computed using the $\mathbf X$ that is such that with probability larger than $1-\delta$
$$|F(x) - \hat F(x)| \leq \sqrt{\frac{\log(2/\delta)}{m}} + \frac{16\log(2/\delta)}{m}  =\phi(\delta,m) := \phi.$$  
\end{lemma}
\begin{proof}[Proof of Lemma~\ref{lem:estF}]
The samples $\mathbf X$ are i.i.d.~distributed according to $F$. This implies that the $\mathbf 1\{X_i \leq x\}$ are i.i.d.~Bernoulli random variables of parameter $F(x)$. Consider the classic estimator 
$\hat F(x) = \frac{1}{m} \sum_{i=1}^m \mathbf 1\{X_i \leq x\}.$
Applying Bernstein's inequality to this estimate implies that with probability $1-\delta$
$$|F(x) - \hat F(x)| \leq 2\sqrt{F(x)(1-F(x))\frac{\log(2/\delta)}{m}} + \frac{16\log(2/\delta)}{m},$$
which implies the result.
\end{proof}

\begin{lemma}[Quantile confidence set]\label{lem:quantconf}
Let $\delta \in (0,1)$. Let $G$ be a distribution, and $\hat G$ be the an estimate of $G$ such that for any $x \in \mathbb R$, on an event of probability $1-\delta$,
$$|\hat G(x) - G(x)| \leq \lambda(\delta):=\lambda.$$
Let $\epsilon \in (0,1)$. With probability $1-2\delta$, we have
$$G^{-1}(\epsilon-\lambda) \leq \hat G^{-1}(\epsilon) \leq G^{-1}(\epsilon+\lambda),$$
where $G^{-1}$ and $\hat G^{-1}$ are the pseudo-inverses of $G$ and $\hat G$.
\end{lemma}
\begin{proof}
Let 
$x = G^{-1}(\epsilon),\quad  x^- = G^{-1}(\epsilon-2\lambda),\quad \mathrm{and} \quad x^+ = G^{-1}(\epsilon+2\lambda).$ 
It holds that $x \in [x^-, x^+],$ since $G$ is a distribution function and is thus increasing. Since $\hat G$ is increasing, it holds that 
$\hat G(x^-) \leq \hat G(x)\leq \hat G(x^+).$ 
Moreover, since the inverse of $G$ exists in the neighbourhood $[\epsilon-\lambda, \epsilon+\lambda]$, we have that $G(x) = \epsilon, \quad  G(x^-) = \epsilon-2\lambda, \quad \mathrm{and} \quad G(x^+) = \epsilon+2\lambda.$

It holds by an union bound that with probability $1-2\delta$, $\hat G(x^-)  \leq G(x^-) + \lambda \quad \mathrm{and} \quad \hat G(x^+) \geq G(x^+) - \lambda.$ This implies that with probability $1-2\delta$, 
$\hat G(x^-)  \leq \epsilon - 2\lambda + \lambda = \epsilon - \lambda \quad \mathrm{and} \quad \hat G(x^+) \geq \epsilon + 2\lambda -  \lambda=\epsilon + \lambda.$
This implies that with probability $1-2\delta$, $$\hat G(x^-)+\lambda \leq \epsilon \leq \hat G(x^+)-\lambda.$$
Since $\hat G$ is increasing as an estimate of a distributions function (and $\hat G^{-1}$ too then), we have that with probability $1-2\delta$,  $x^- \leq \hat G^{-1}(\epsilon) \leq x^+.$
This concludes the proof.
\end{proof}

By Lemma~\ref{lem:estF}, we have with probability $1-\delta$, $|\hat F(u) - F(u)| \leq \phi.$
Let $\epsilon = F(u)$ and $\hat \epsilon  = \hat F(u)$. The previous equation implies that with probability $1-\delta$, $|\hat \epsilon - \epsilon| \leq \phi.$

By Lemma~\ref{lem:quantconf} and Assumption~\ref{ass:estFh}, we have with probability $1-\delta$, $F_h^{-1}(\hat \epsilon-\psi) \leq \hat F_h^{-1}(\hat\epsilon) \leq F_h^{-1}(\hat \epsilon+\psi),$ where $F_h^{-1}$ is the pseudo-inverse of $F_h$, which implies since $F_h^{-1} = h\circ F^{-1}$, that with probability $1-\delta$, $h\circ F^{-1}(\hat \epsilon-\psi) \leq \hat F_h^{-1}(\hat \epsilon) \leq h\circ F^{-1}(\hat \epsilon+\psi),$ where $F^{-1}$ is the pseudo-inverse of $F$.

The previous equations imply, since $h\circ F^{-1}$ is increasing that with probability $1-2\delta$, $h\circ F^{-1}(\epsilon-\psi - \phi) \leq \hat F_h^{-1}\circ \hat F(x) \leq h\circ F^{-1}( \epsilon+ \psi +\phi),$ which implies that with probability $1-2\delta$, $h\circ F^{-1}(F(u)-\psi - \phi) \leq \hat F_h^{-1}\circ \hat F(u) \leq h\circ F^{-1}( F(u) + \psi +\phi).$

\subsection{Proof of Theorem~\ref{th:cocobleu2}}\label{pr:cocobleu2}

In this proof, in order to avoid notational heaviness, we write for convenience $\mathbf X$ instead of $\mathbf X^Z$ and $\mathbf Y$ instead of $\mathbf Y^Z$.  We also write $F$ for $F_Z$ and $F_h$ for $F_{h,Z}$ (and the same simplification for the estimators).

We first state the following lemma.

\begin{lemma}[Relation between the graph of $h$ and the distributions]\label{lem:graph}
Let $u \in \mathbb R$. Assume that $F$ is strictly increasing in $u$. Then
$$h(u) = F_h^{-1}(F(u)),$$
where $F^{-1}$ is the pseudo-inverse of $F$.
\end{lemma}
\begin{proof}
Let $h^{-1}$ and $F^{-1}$ be the pseudo-inverses of $F,h$.
By definition $\mathbb P(X \leq u) = F(u),$ and $\mathbb P(h(X) \leq u) =  F_h(u).$ Moreover 
$\mathbb P(h(X) \leq u) = \mathbb P(X \leq h^{-1}(u)) = F(h^{-1}(u)),$ where $h^{-1}$ is the pseudo-inverse of $h$.

We have 
$F_h^{-1}(F(u)) = (F \circ h^{-1})^{-1} \circ F(u) = h \circ F^{-1} \circ F(u) = h(u),$ 
since $F$ is strictly increasing in $u$. This concludes the proof.
\end{proof}

Lemma~\ref{lem:graph} implies since $F$ is strictly increasing in $u$ (since $F^{-1}$ is increasing and H\"older continuous) that $F_h^{-1}\circ F(u) = h(u)$.

From Theorem~\ref{th:cocobleu}, we have that with probability $1-2\delta$,
$h\circ F^{-1}(F(u)-\psi - \phi) \leq \hat F_h^{-1}\circ \hat F(u) \leq h\circ F^{-1}( F(u) + \psi +\phi).$ 
Since $F^{-1}$ is $(\beta,M)-$H\"older on $[F(u)-\psi - \phi, F(u) + \psi + \phi]$, this implies that with probability $1-2\delta$
$$h(u-M(\psi +\phi)^{\beta}) \leq \hat F_h^{-1}\circ \hat F(u) \leq h(u + M(\psi +\phi)^{\beta}),$$
and since $h$ is $(\alpha,L)-$H\"older on $[u-M(\psi +\phi)^{\beta}, u + M(\psi +\phi)^{\beta}]$, this implies that with probability $1-2\delta$
$$h(u)-LM^{\alpha}(\psi +\phi)^{\alpha\beta} \leq \hat F_h^{-1}\circ \hat F(u) \leq h(u) + LM^{\alpha}(\psi +\phi)^{\alpha\beta}.$$
This concludes the proof.

\section{Additional Experiments}\label{s:expes2}

\subsection{Simulations}

In the simulations, we do not use any contextual variables ($d = 0$) in order to focus on the gain due to our method. 

\subsubsection{Simulations with different sample size and distributions $\xi$}\label{sss:sim1}

We simulate $m$ data $X_i \sim U([-5,5])$ that are i.i.d.~(here $f$ is a uniform density on $[-5,5]$), and $n=m$ data $Y_j \sim g = f_h * \xi$, independently of the $X_i$, for different distributions $\xi$ and $m=n$. The functions $h$ we consider in these simulations is $h(x) = x|x|/4$. 

We estimate the link function $h$ by 
$\hat h = \hat F_h^{-1}\circ \hat F,$ 
using the distribution of $\xi$ to deconvolve $g$ from $\xi$. The results of these simulations, i.e.~plots of $h$ and of a realisation of $\hat h$, are provided in Figure~\ref{fig:exp1}, for various noises $\xi$ and sample sizes.

Without surprises, the larger $m=n$, the better the estimator $\hat h$. A less intuitive fact, is that the more heavy tailed the distribution, the the better the deconvolution (the performances when $\xi$ is a Student distribution are better than for the Gaussian and uniform cases). This comes from the fact that heavy tailed noise distributions are easier to deconvolve than light tailed noises, as we mentioned below Assumption~\ref{ass:estFh}. 

\begin{center}
\begin{figure*}[!htb]
\begin{minipage}{4.2cm}
\includegraphics[width=4.5cm]{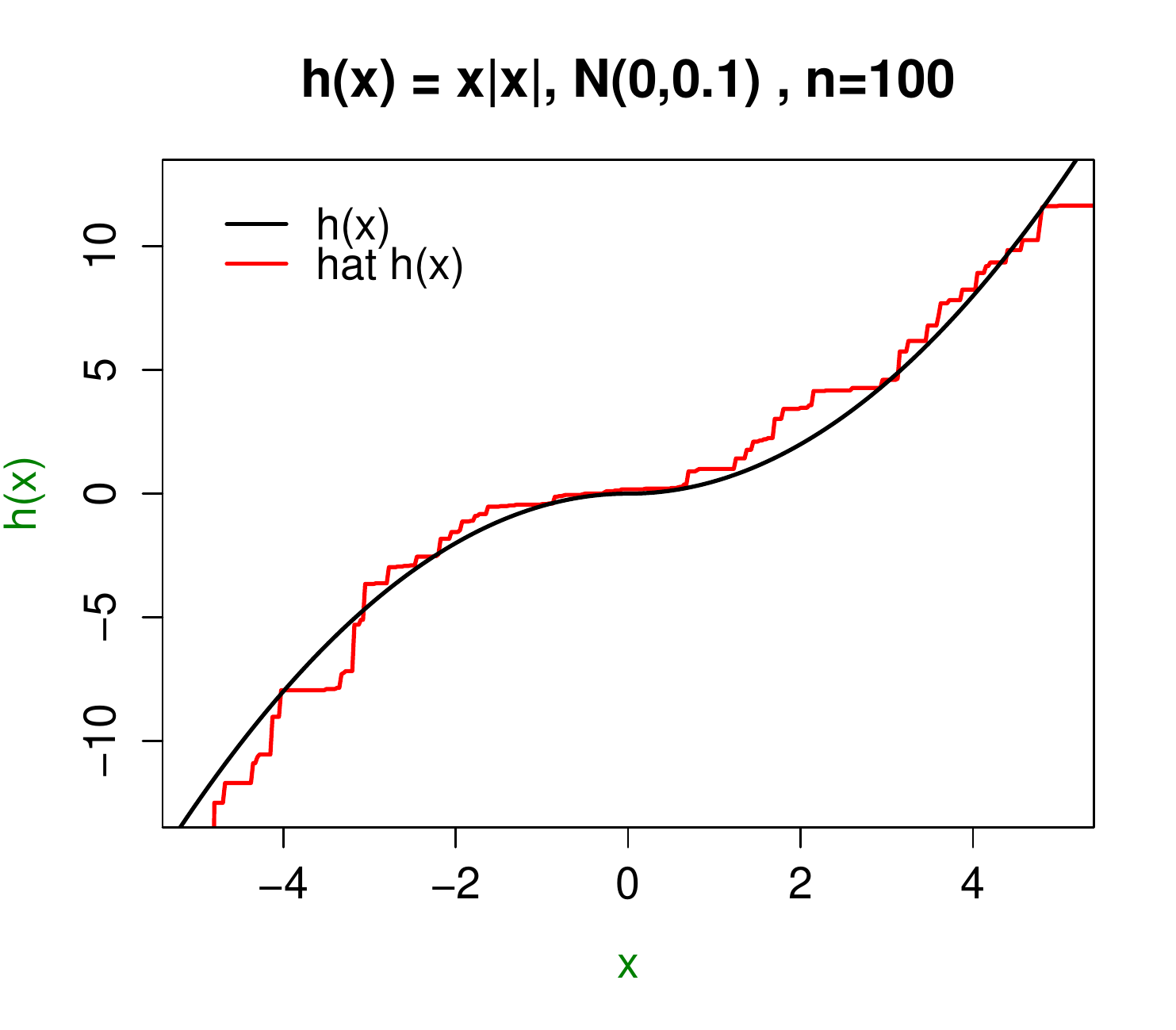}
\end{minipage}
\begin{minipage}{4.2cm}
\includegraphics[width=4.5cm]{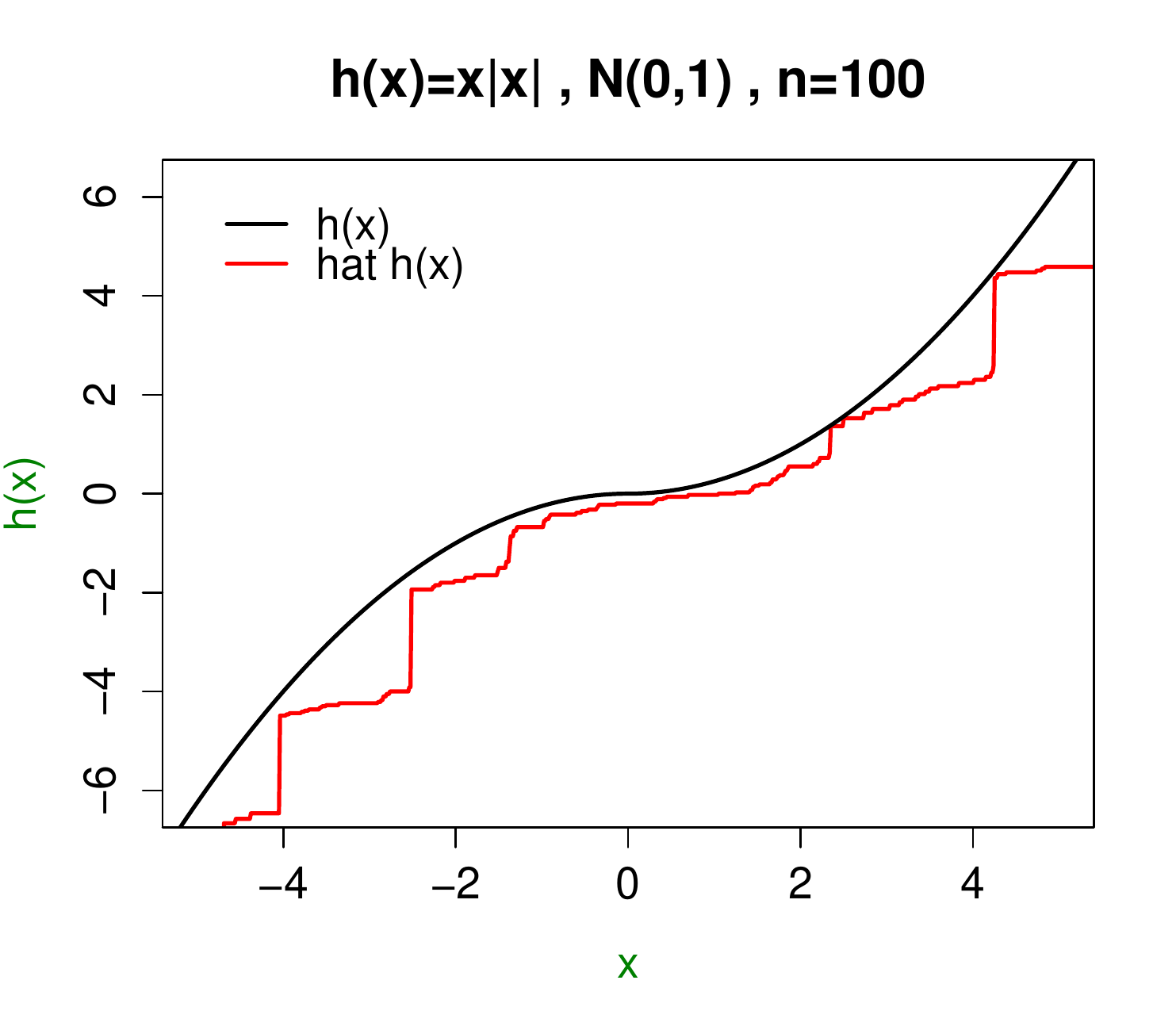}
\end{minipage}
\begin{minipage}{4.2cm}
\includegraphics[width=4.5cm]{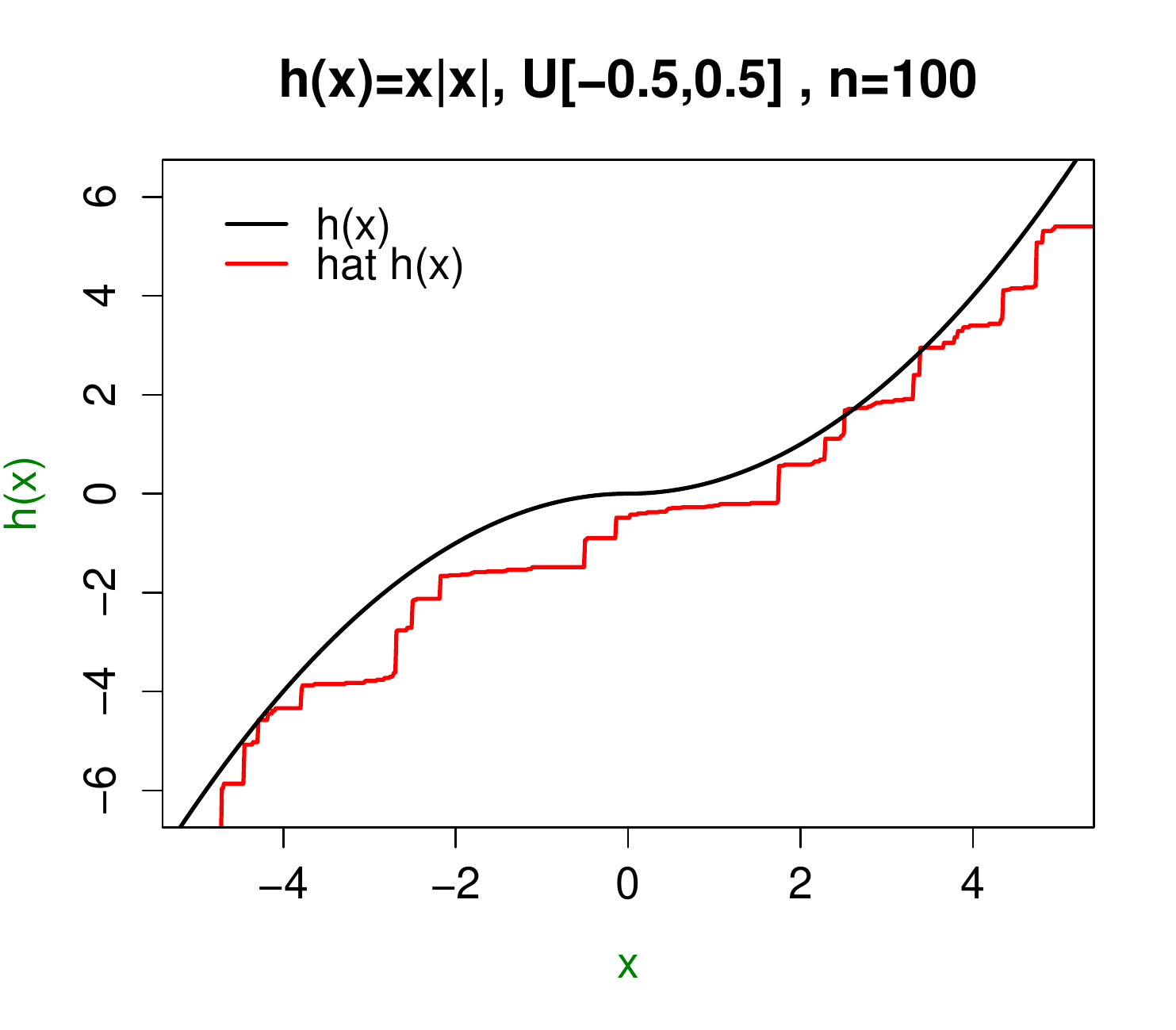}
\end{minipage}
\begin{minipage}{4.2cm}
\includegraphics[width=4.5cm]{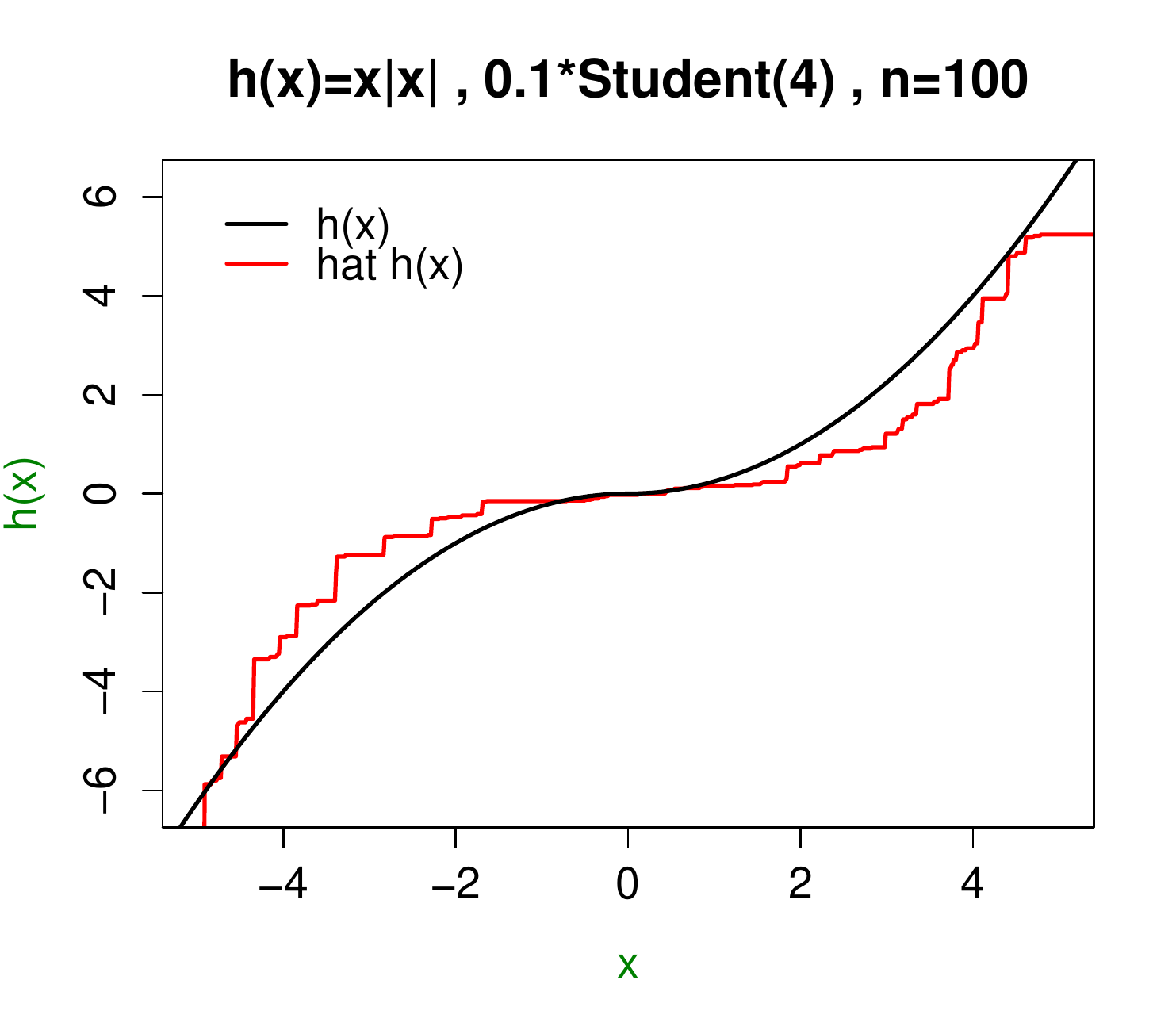}
\end{minipage}

\begin{minipage}{4.2cm}
\includegraphics[width=4.5cm]{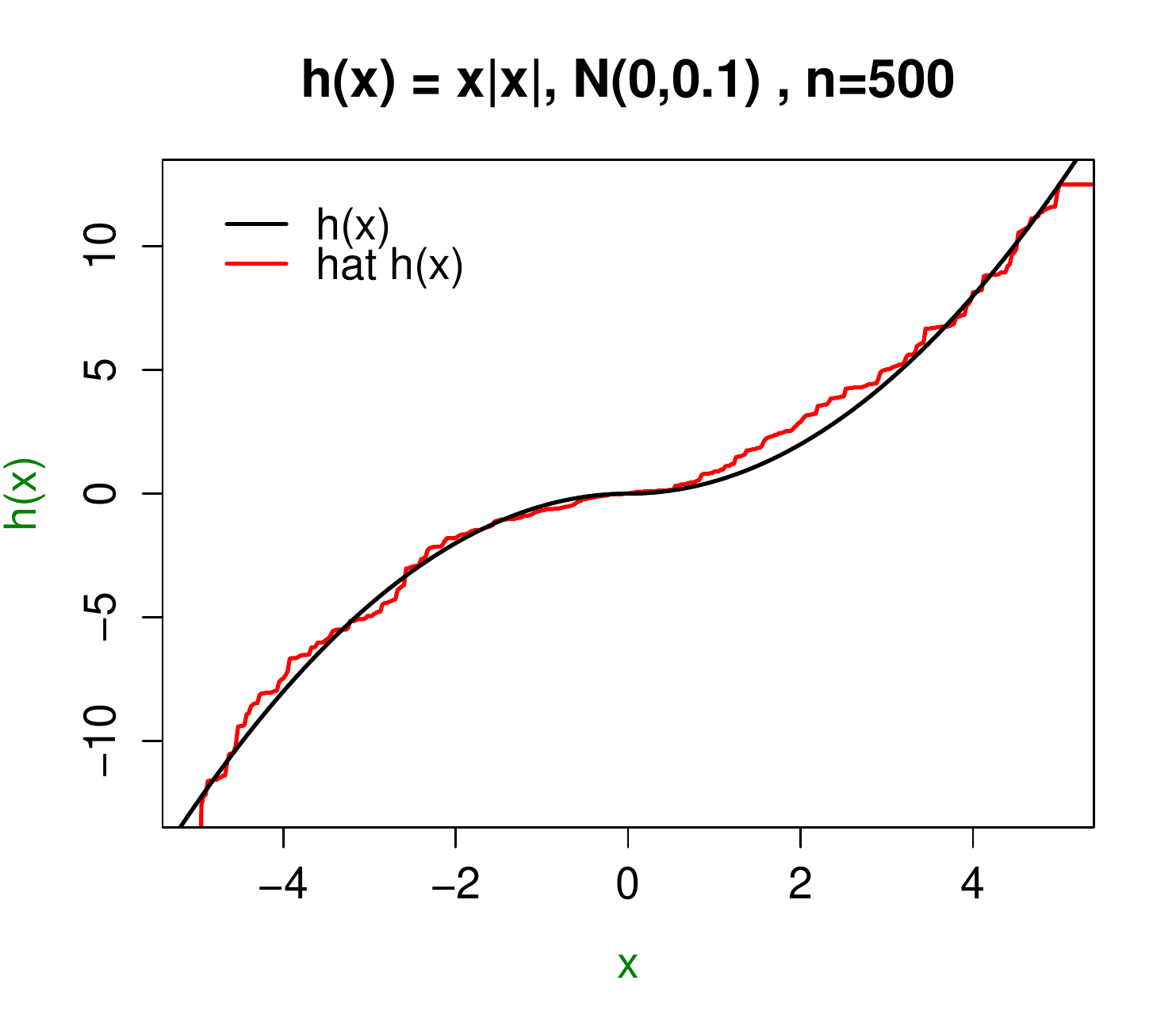}
\end{minipage}
\begin{minipage}{4.2cm}
\includegraphics[width=4.5cm]{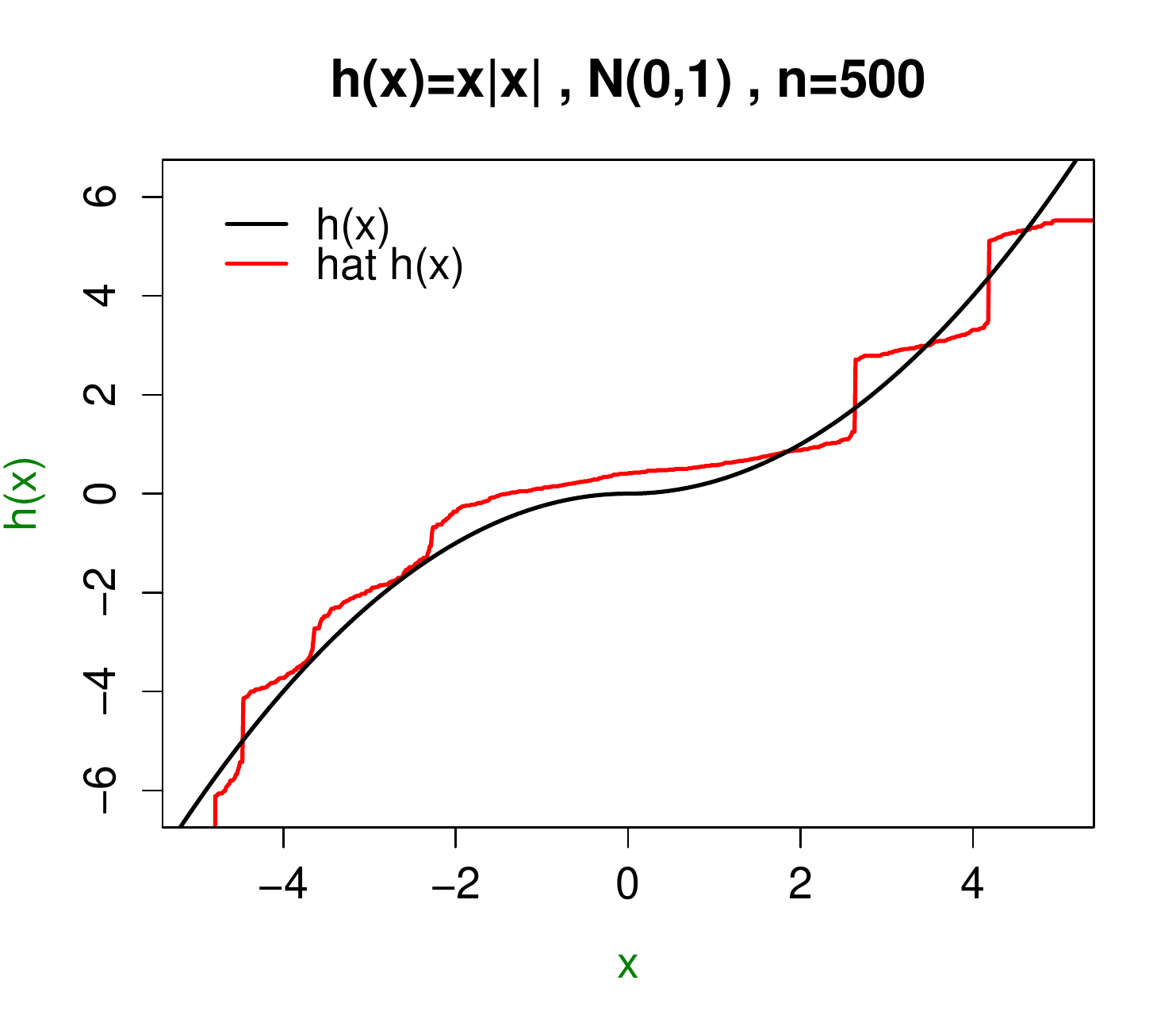}
\end{minipage}
\begin{minipage}{4.2cm}
\includegraphics[width=4.5cm]{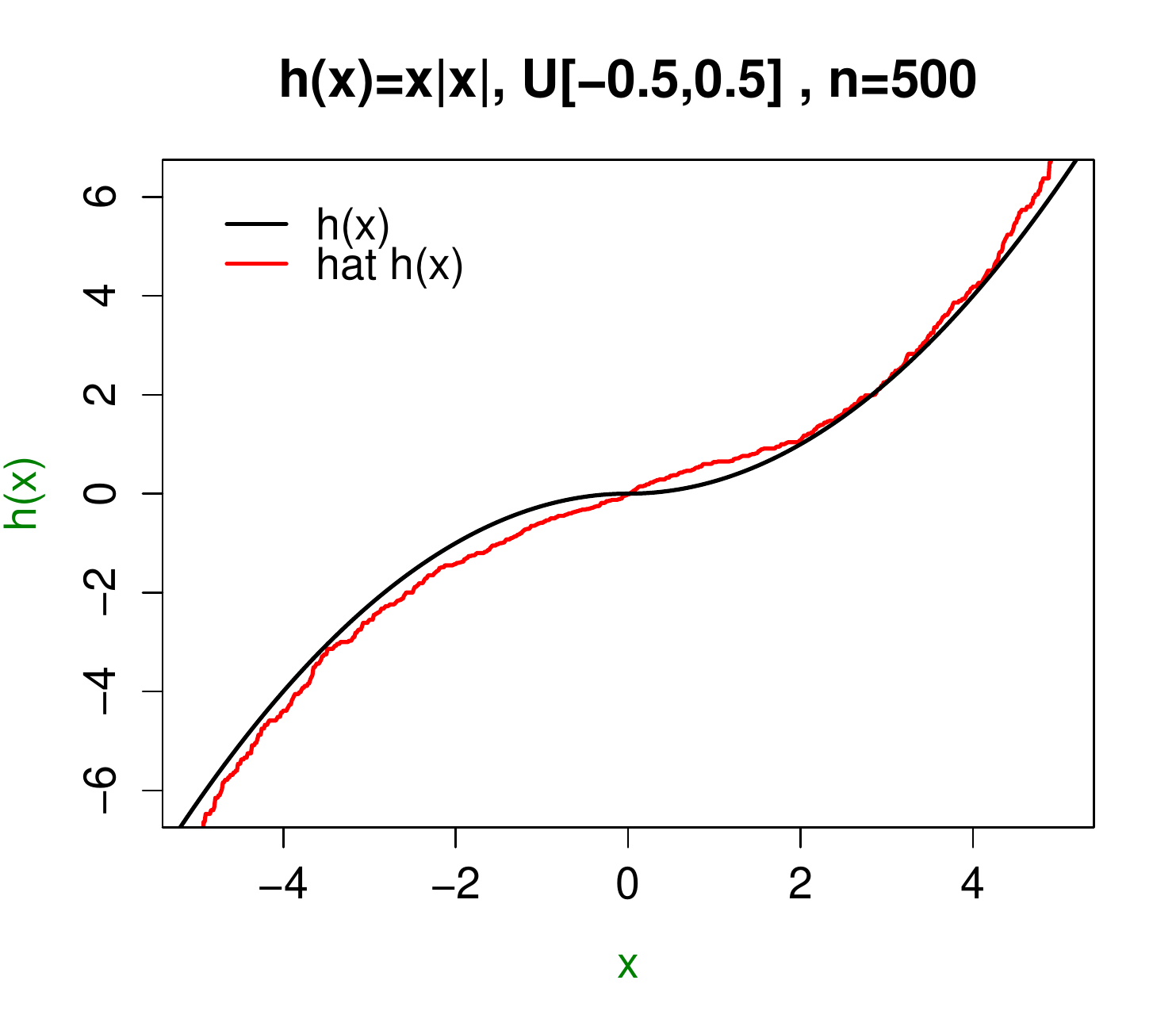}
\end{minipage}
\begin{minipage}{4.2cm}
\includegraphics[width=4.5cm]{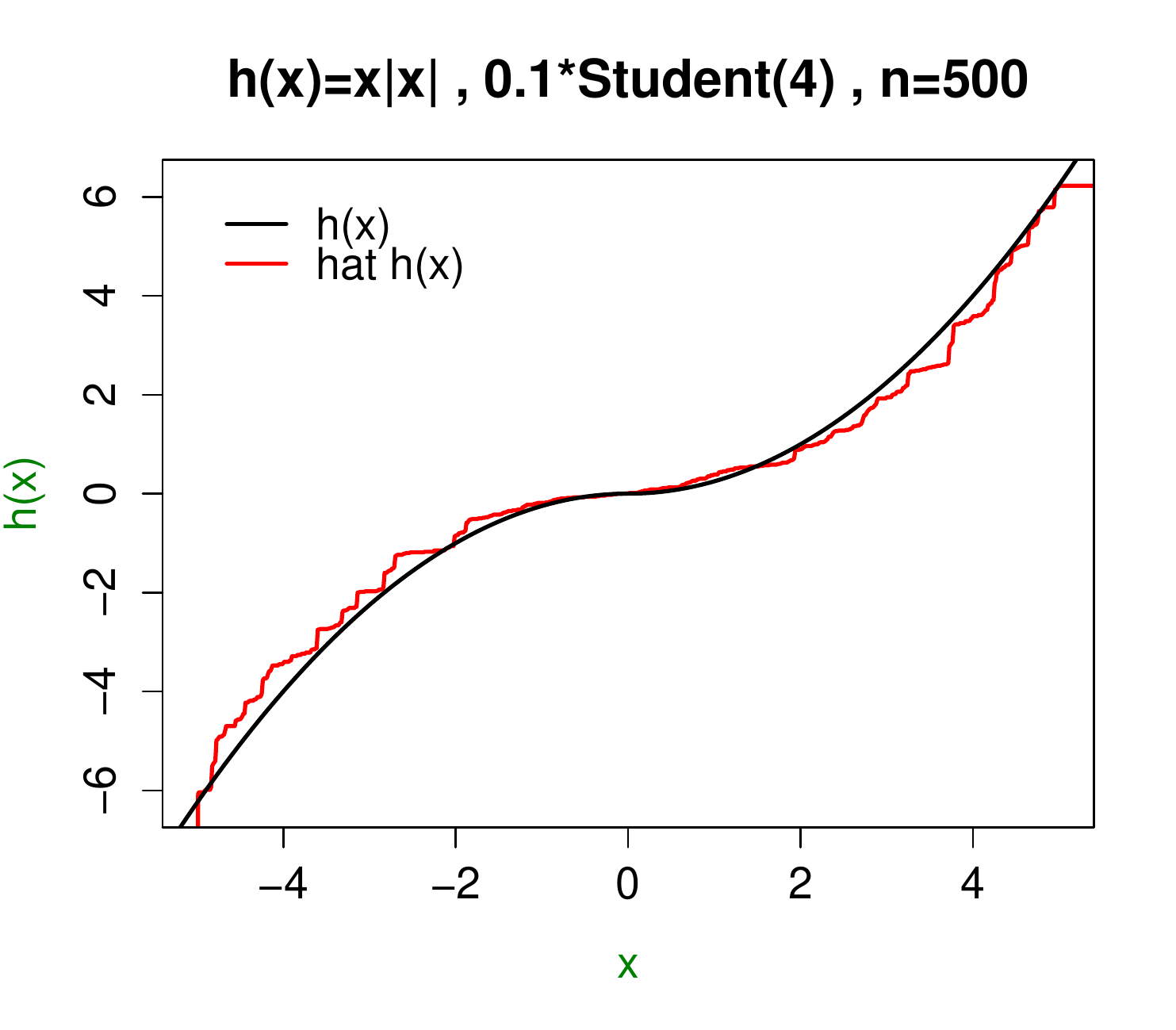}
\end{minipage}

\begin{minipage}{4.2cm}
\includegraphics[width=4.5cm]{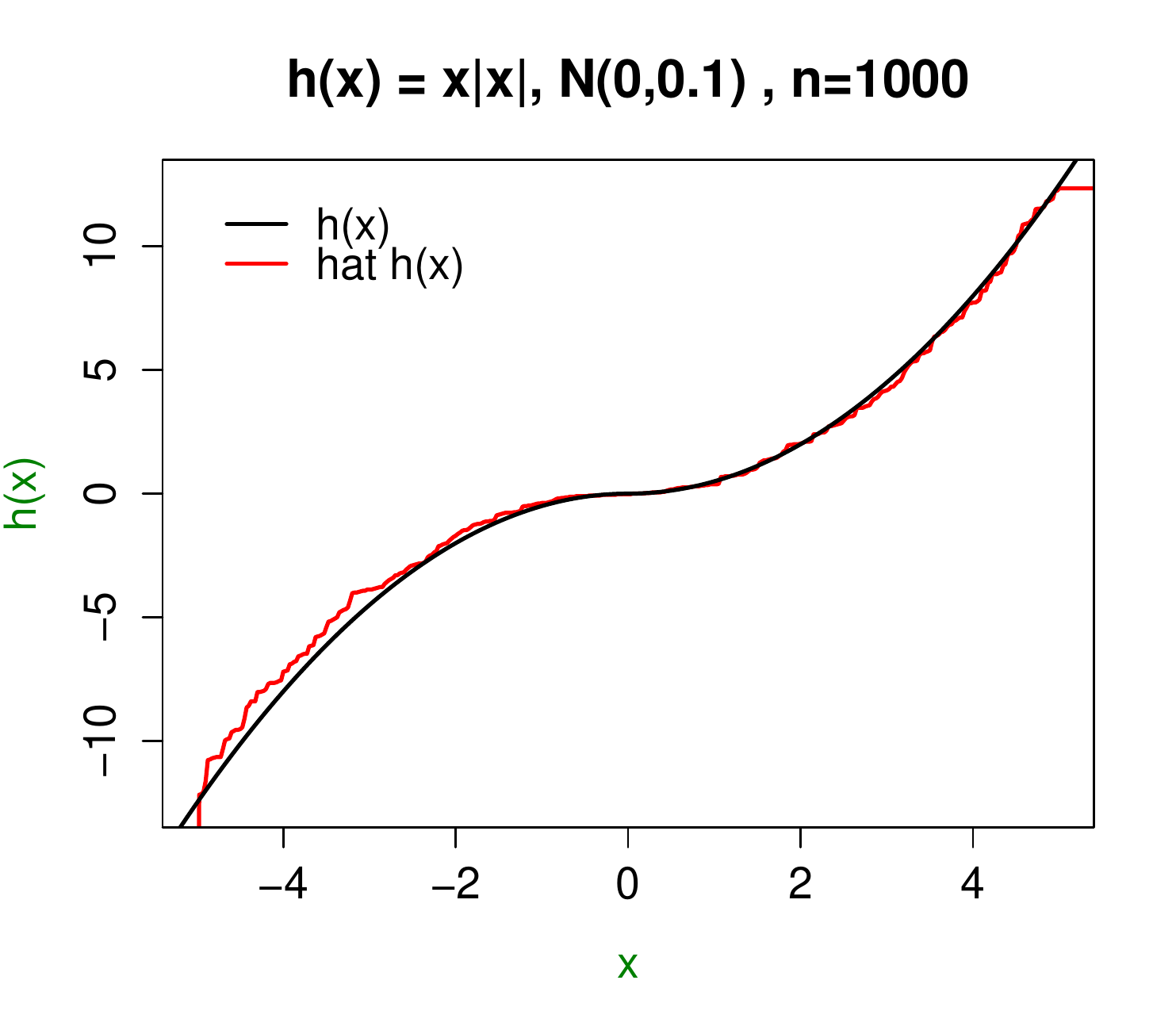}
\end{minipage}
\begin{minipage}{4.2cm}
\includegraphics[width=4.5cm]{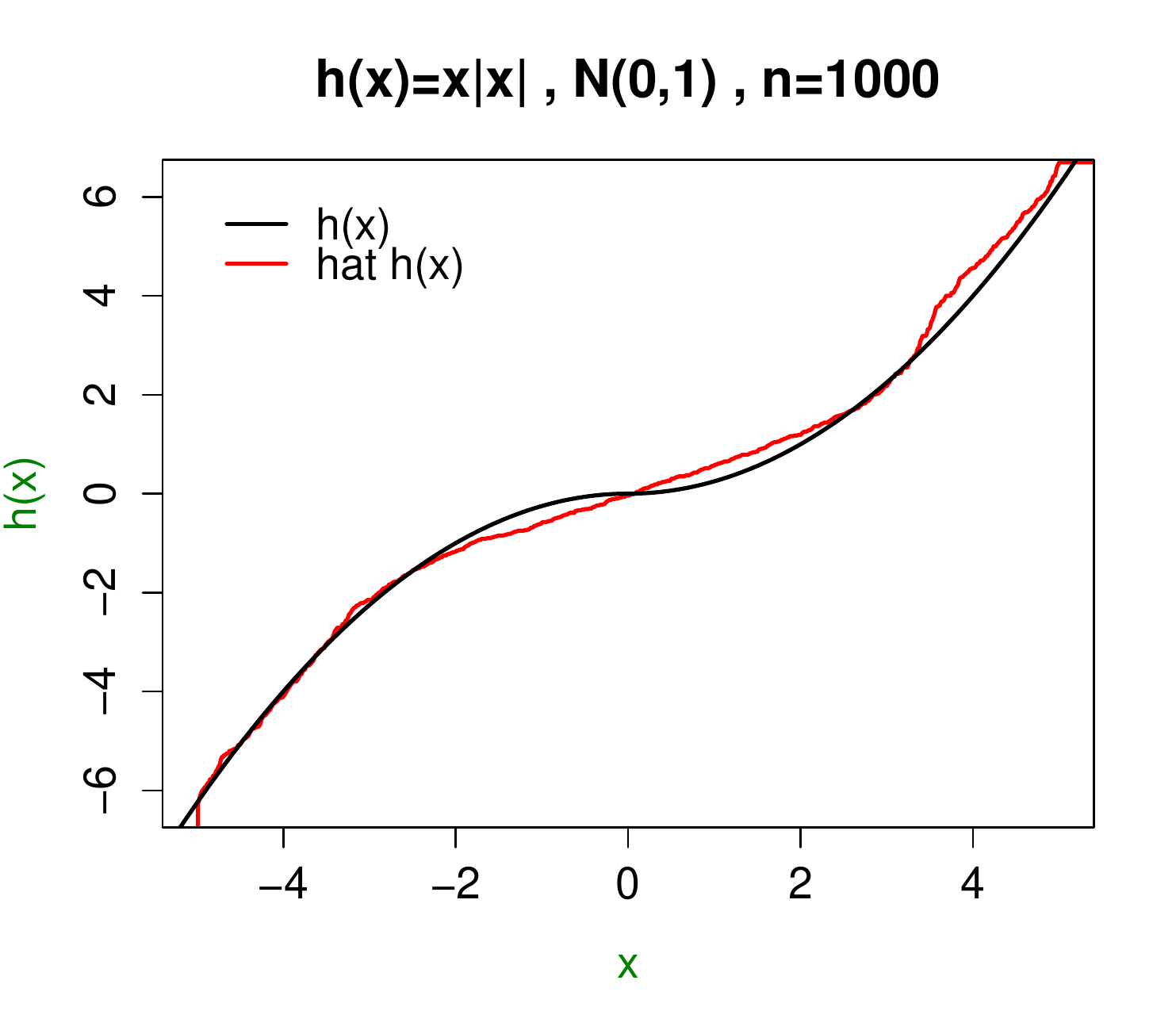}
\end{minipage}
\begin{minipage}{4.2cm}
\includegraphics[width=4.5cm]{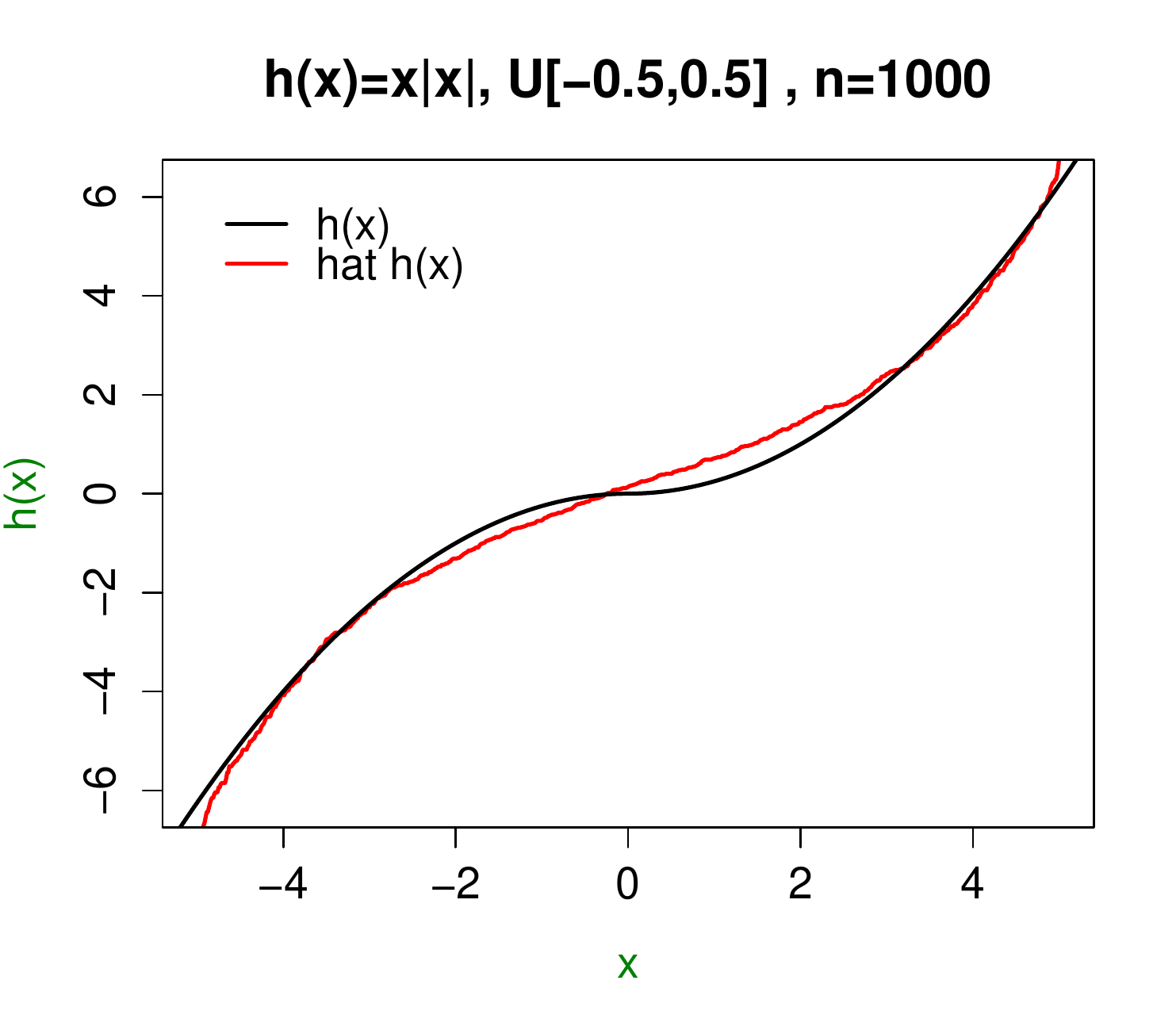}
\end{minipage}
\begin{minipage}{4.2cm}
\includegraphics[width=4.5cm]{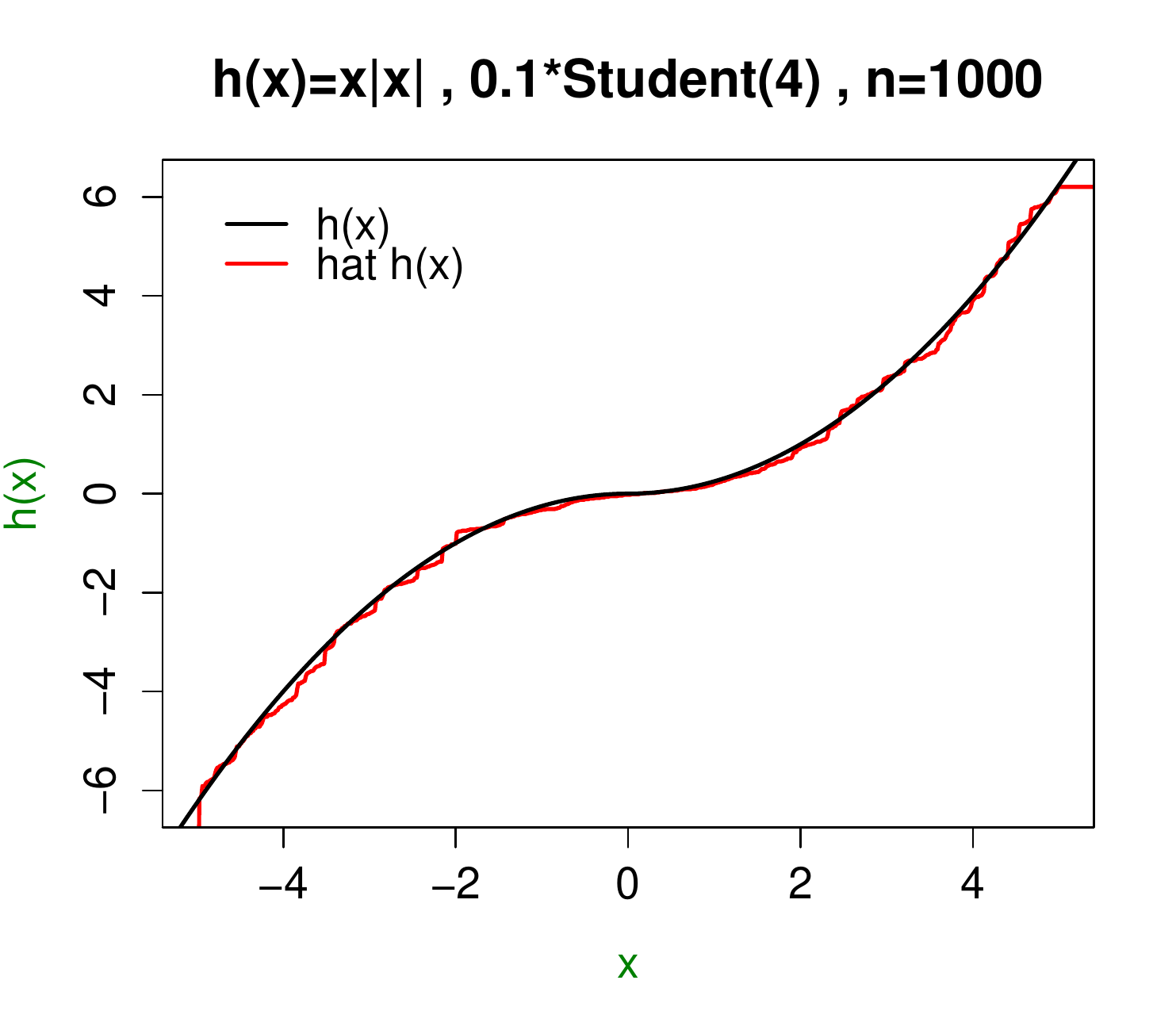}
\end{minipage}
\caption{A realisation of $\hat h$ (red) when $h:x \rightarrow x|x|/4$, for different values of $m=n$, and different types of noise. Left to right: $\xi$ is (i) a $\mathcal N(0,0.1)$, (ii) a $\mathcal N(0,1)$, (iii) a $U([-0.5,0.5])$ and (iv) a $0.1 Student(4)$. Up to down: $m=n$ takes value (i) $100$, (ii) $500$ and (iii) $1000$.} \label{fig:exp1}
\end{figure*}
\end{center}

\subsubsection{Simulations with a function $h$ that has a discontinuity}

We now consider the same setting as before (in Subsubsection~\ref{sss:sim1}) but now we consider a function $h$ that has a discontinuity. The results are displayed in Figure~\ref{fig:exp2}.

The function $h(x) = x|x|/4$ of Figure~\ref{fig:exp1} is easier to reconstruct than the function $h$ of Figure~\ref{fig:exp2}, which has a discontinuity and that thus does not satisfy the assumptions of Theorem~\ref{th:cocobleu2}. The estimator is however not far from the original function in particular in Student noise, or at some distance of the discontinuity.

\begin{center}
\begin{figure*}[!htb]
\begin{minipage}{4.2cm}
\includegraphics[width=4.5cm]{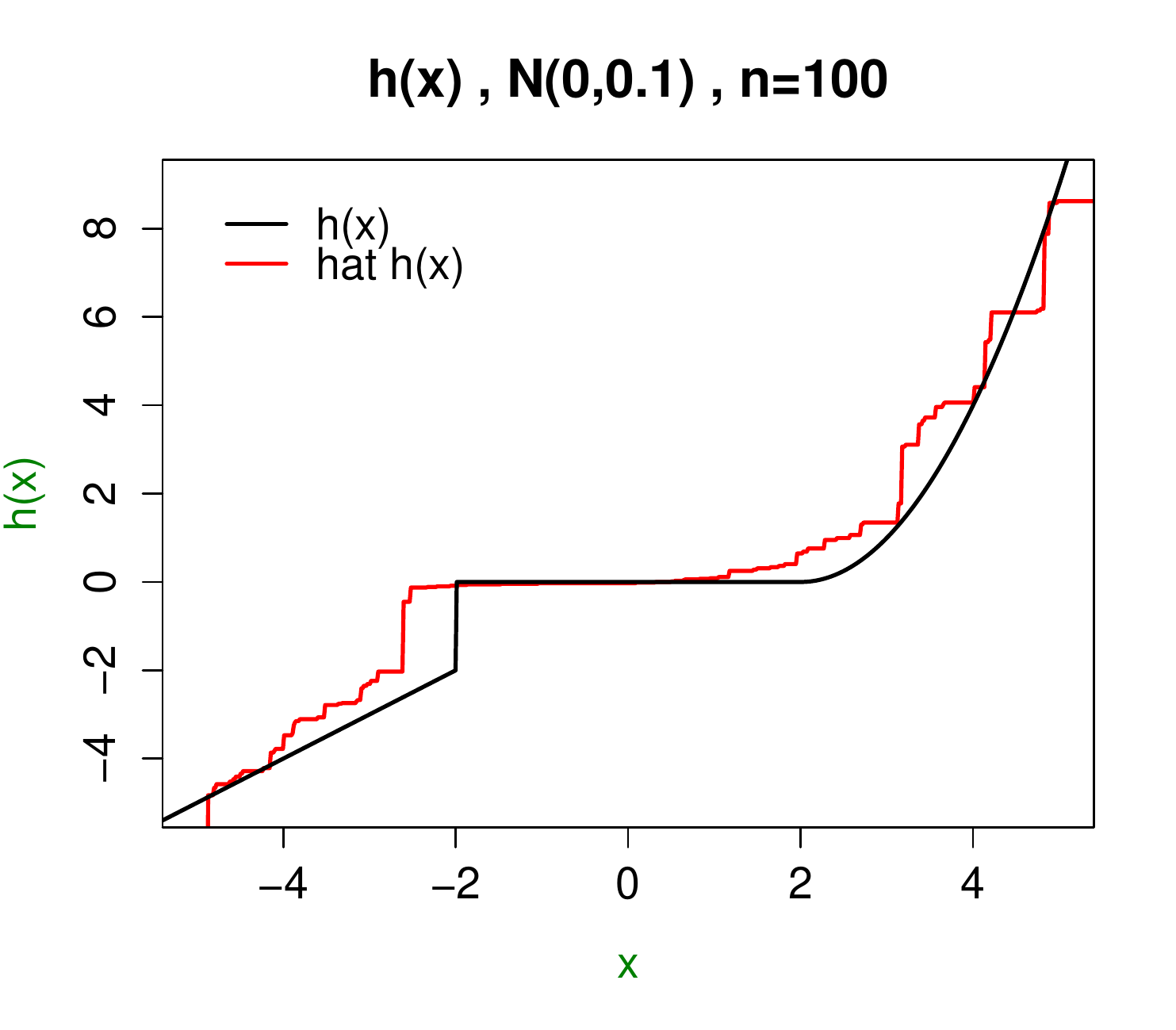}
\end{minipage}
\begin{minipage}{4.2cm}
\includegraphics[width=4.5cm]{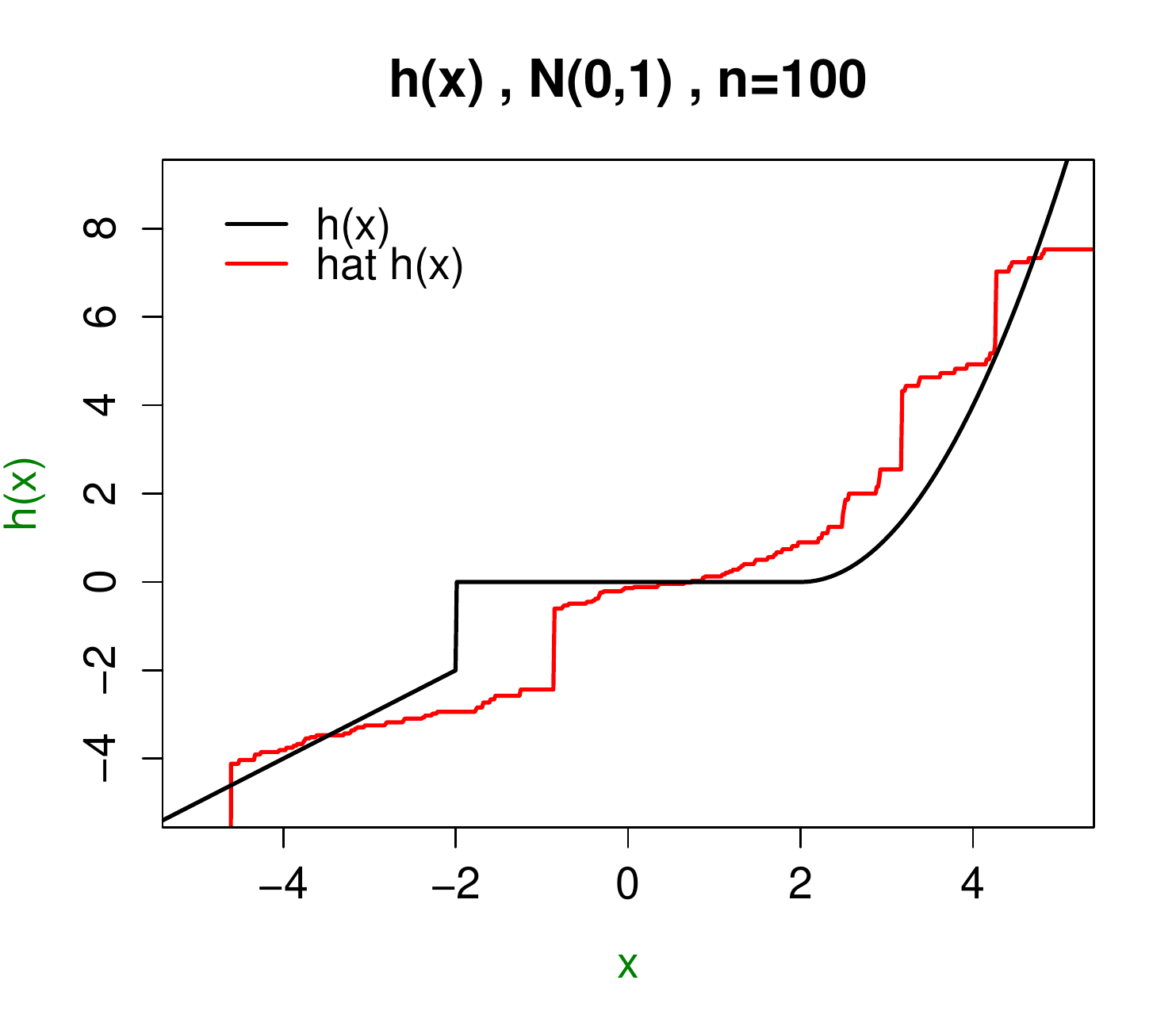}
\end{minipage}
\begin{minipage}{4.2cm}
\includegraphics[width=4.5cm]{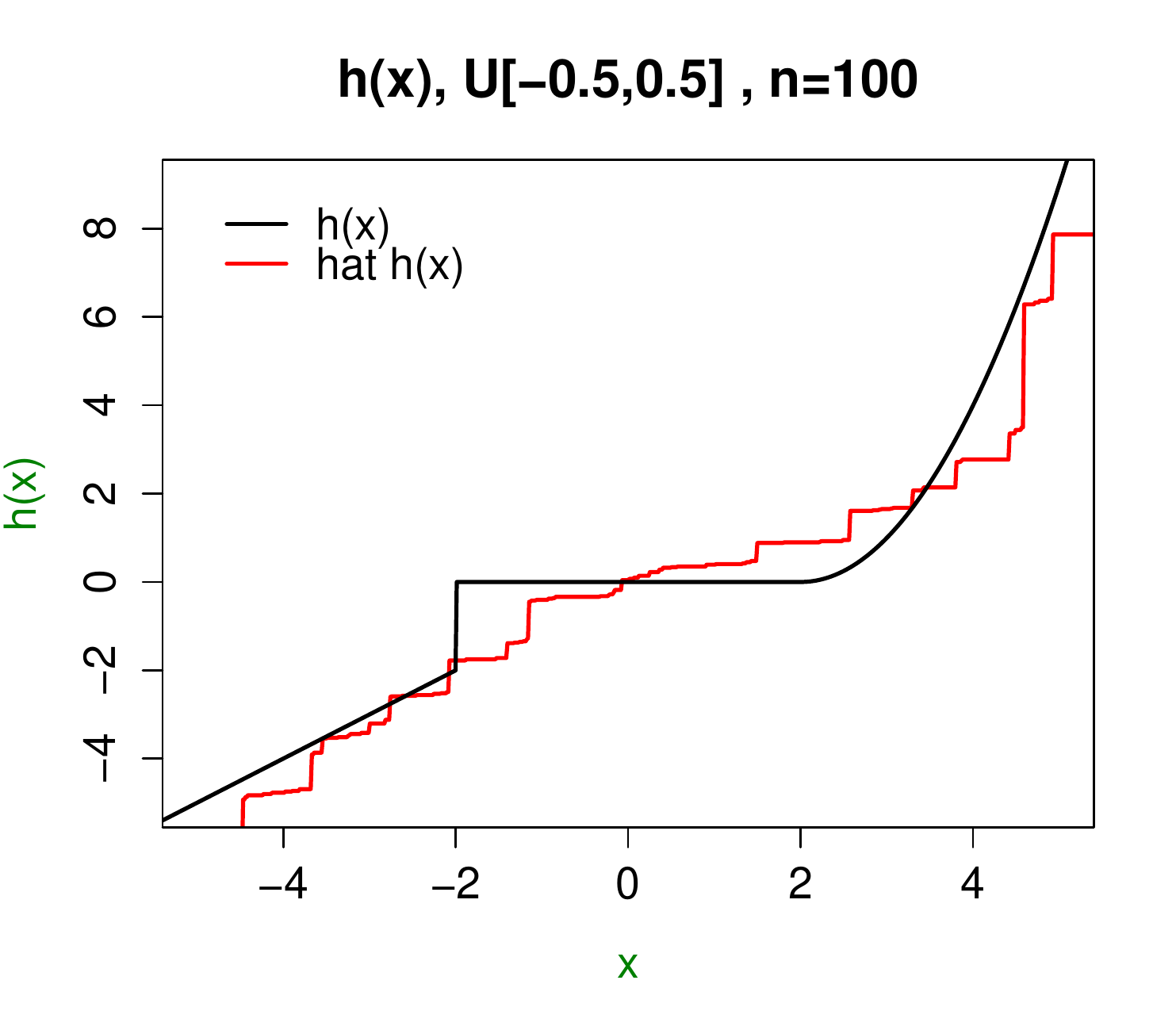}
\end{minipage}
\begin{minipage}{4.2cm}
\includegraphics[width=4.5cm]{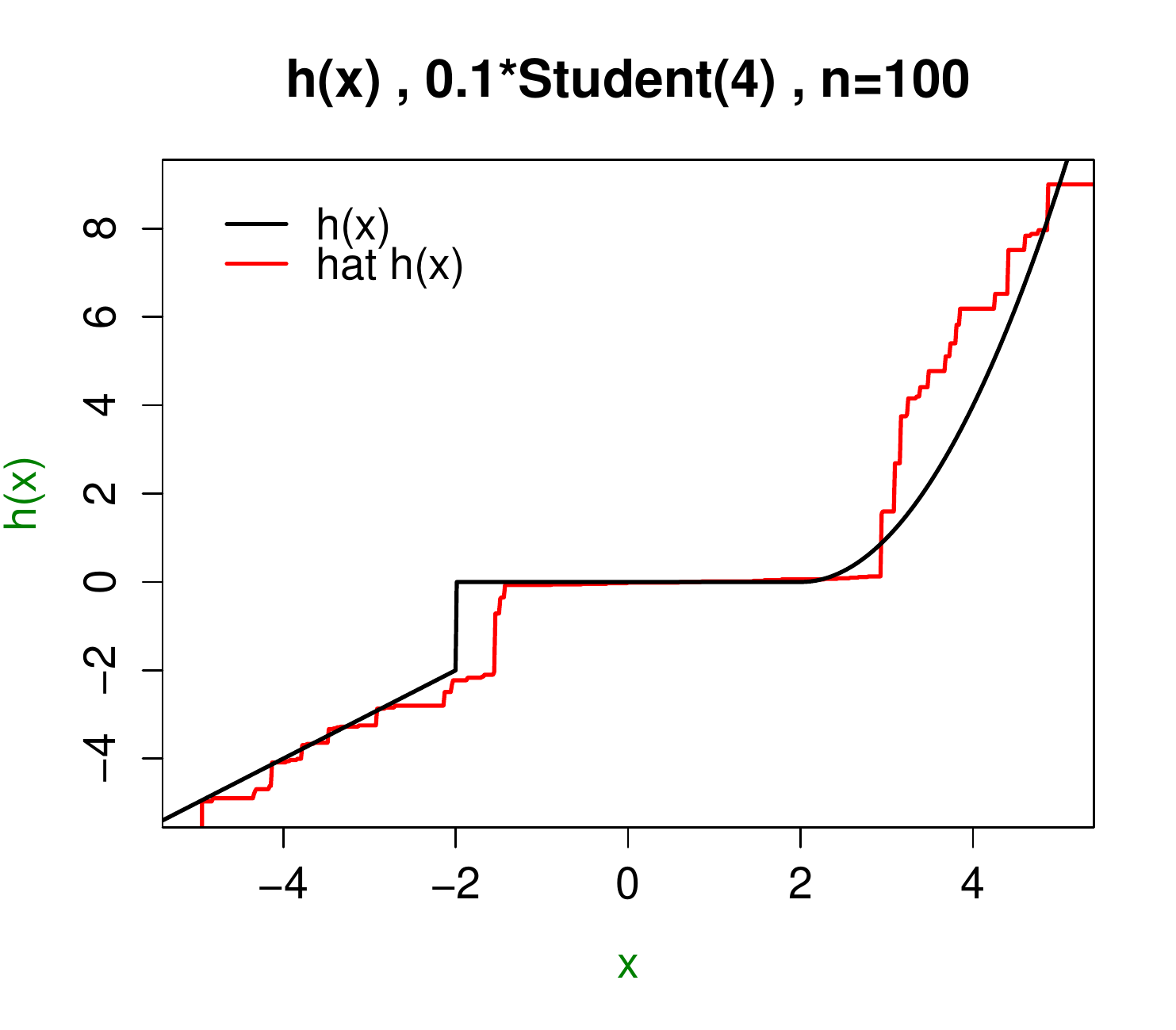}
\end{minipage}

\begin{minipage}{4.2cm}
\includegraphics[width=4.5cm]{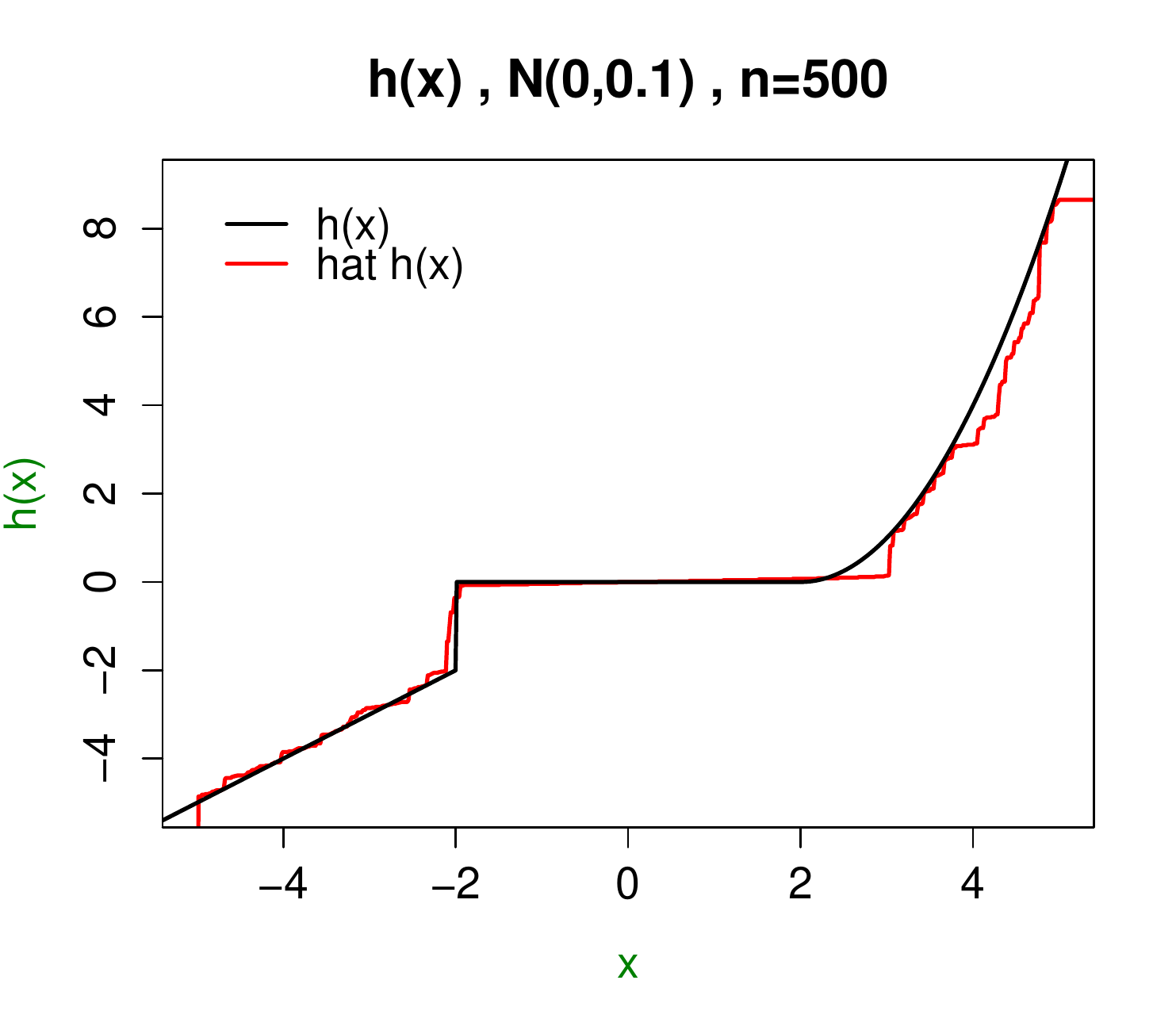}
\end{minipage}
\begin{minipage}{4.2cm}
\includegraphics[width=4.5cm]{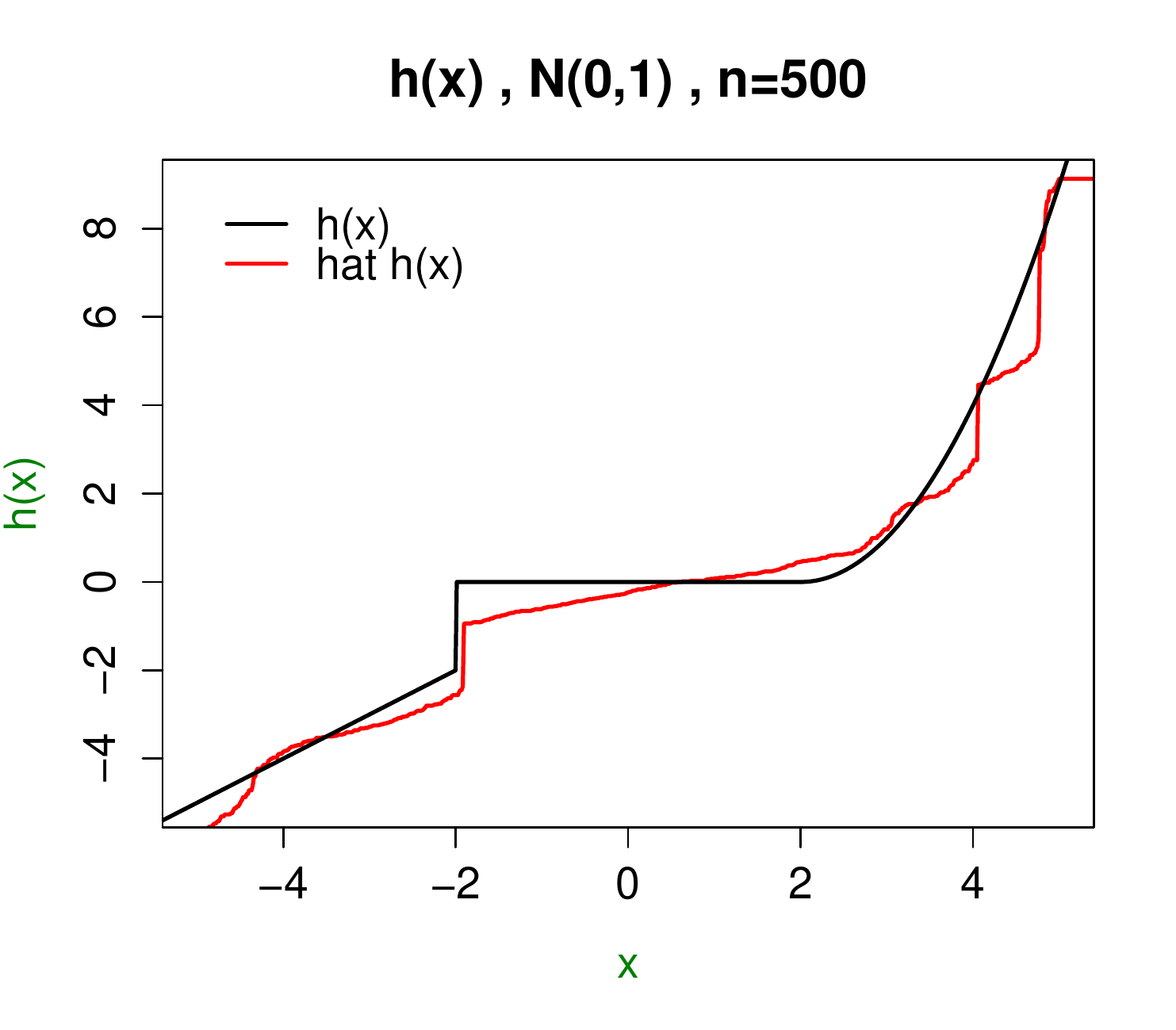}
\end{minipage}
\begin{minipage}{4.2cm}
\includegraphics[width=4.5cm]{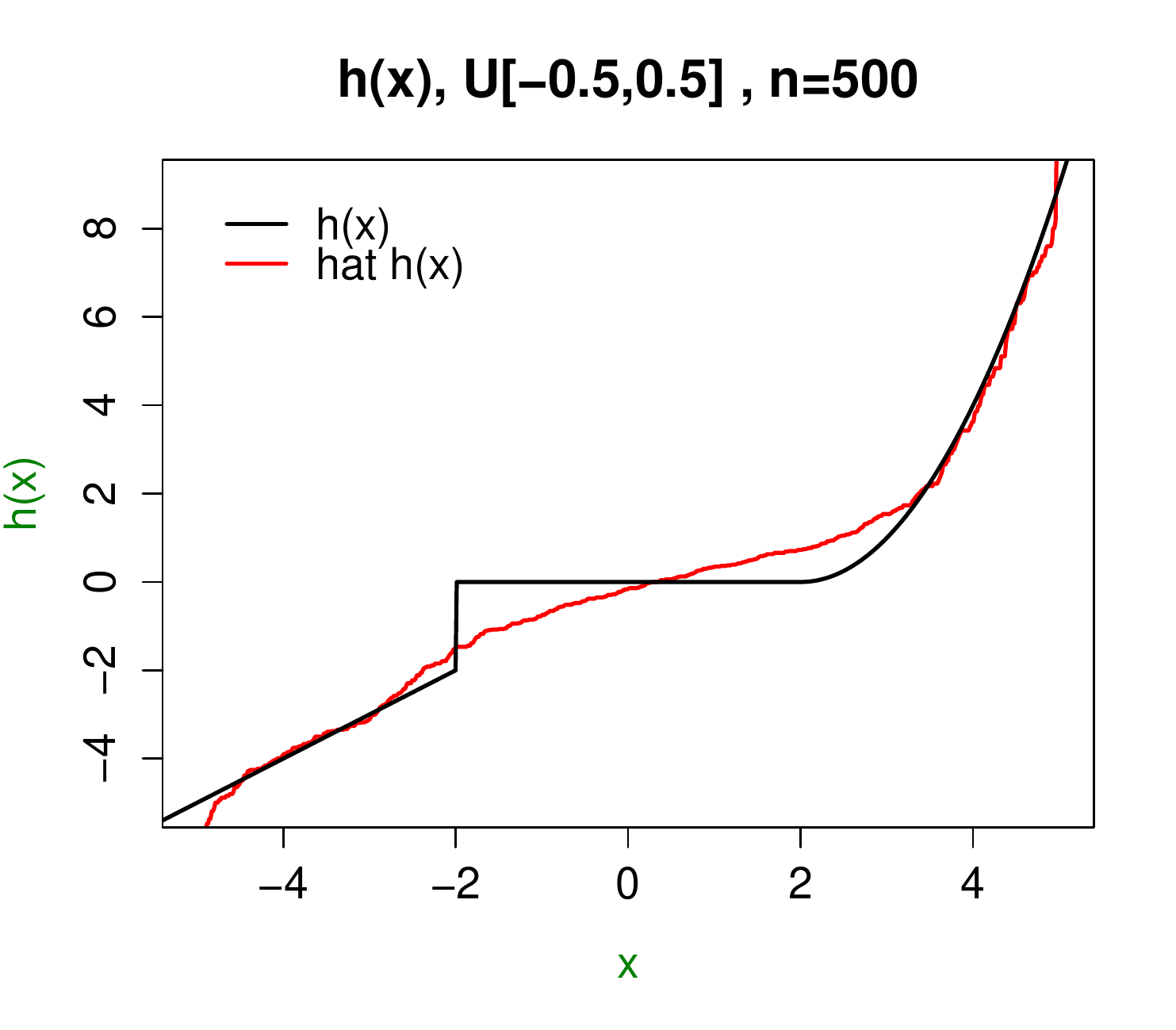}
\end{minipage}
\begin{minipage}{4.2cm}
\includegraphics[width=4.5cm]{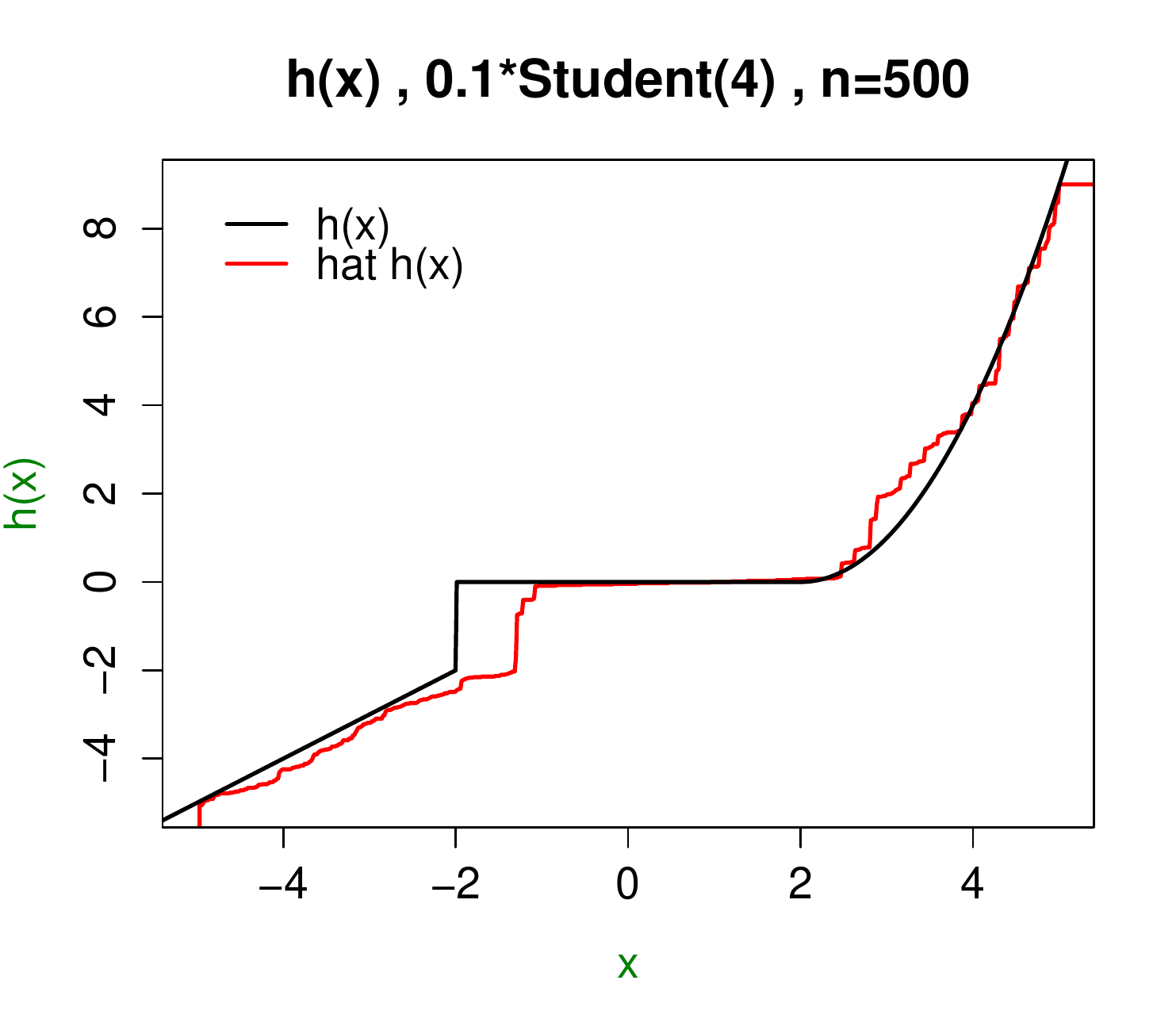}
\end{minipage}

\begin{minipage}{4.2cm}
\includegraphics[width=4.5cm]{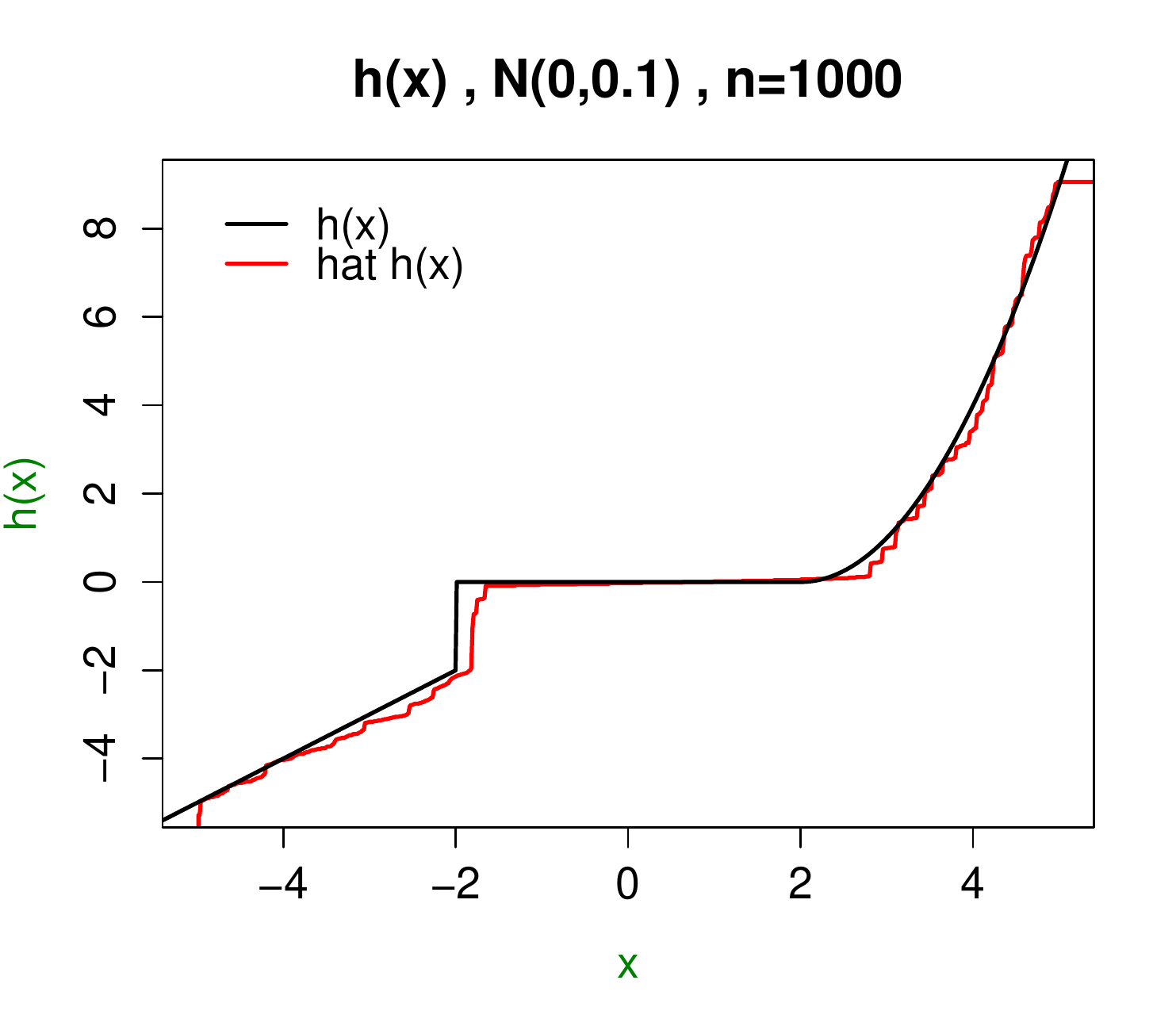}
\end{minipage}
\begin{minipage}{4.2cm}
\includegraphics[width=4.5cm]{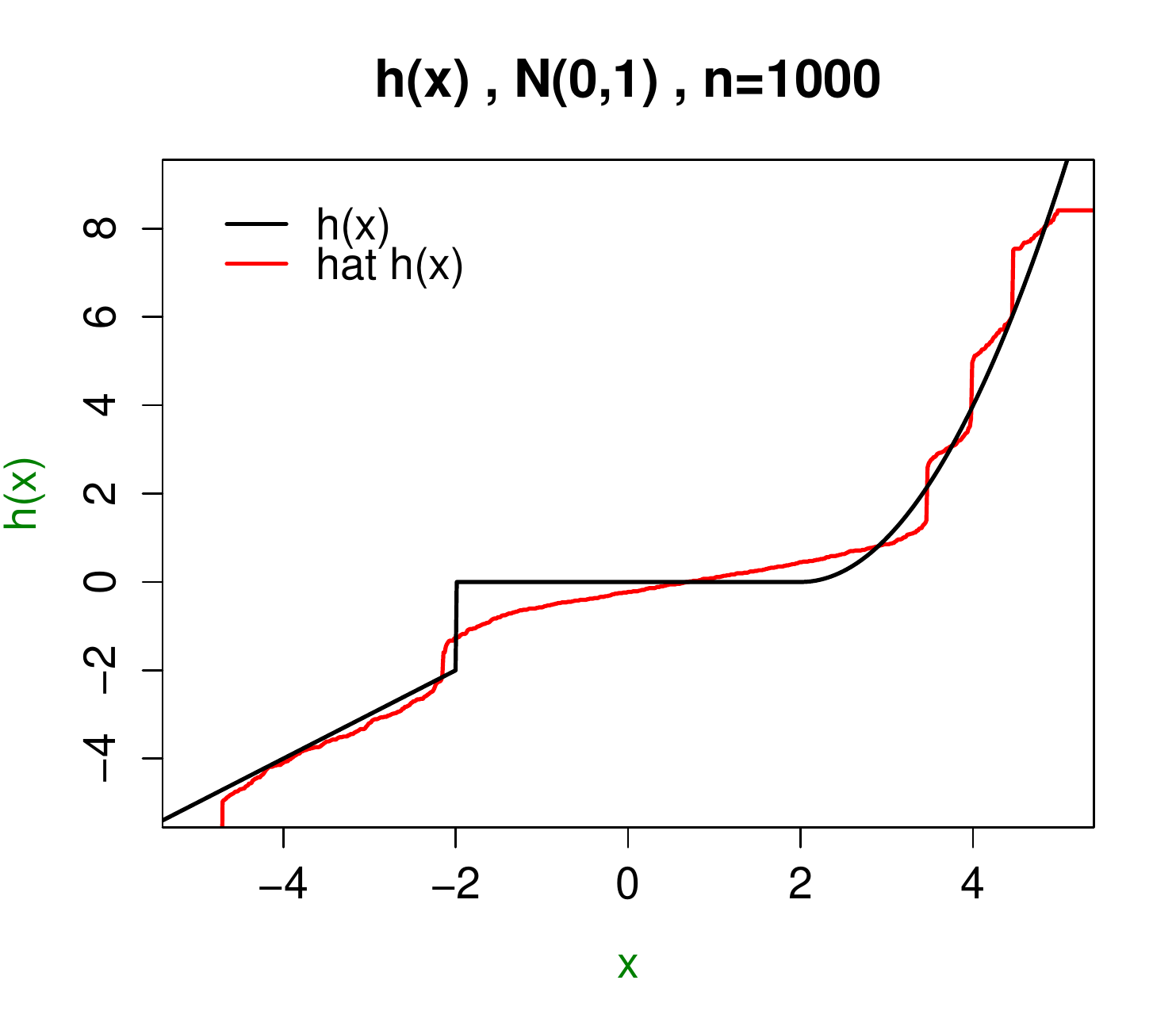}
\end{minipage}
\begin{minipage}{4.2cm}
\includegraphics[width=4.5cm]{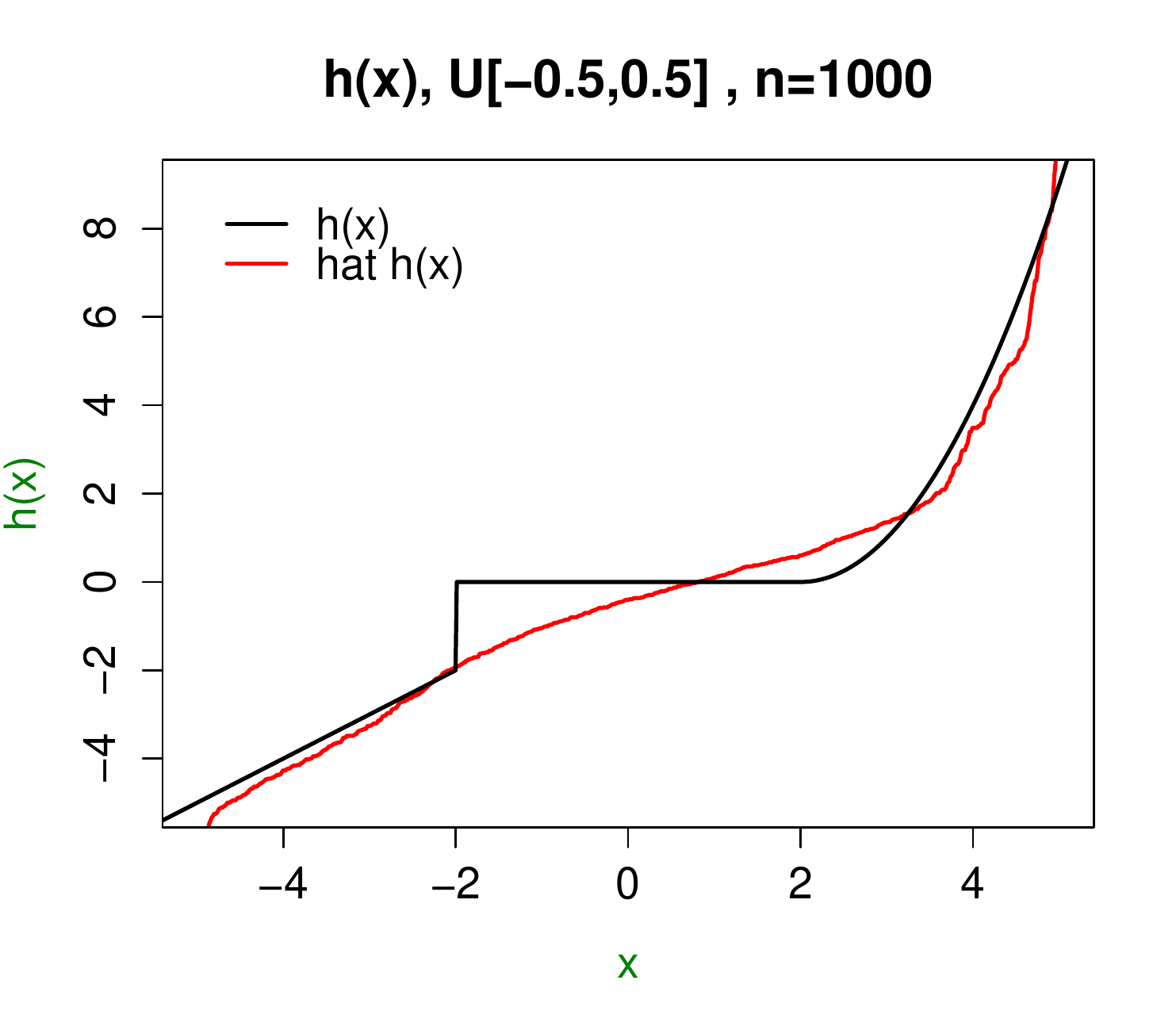}
\end{minipage}
\begin{minipage}{4.2cm}
\includegraphics[width=4.5cm]{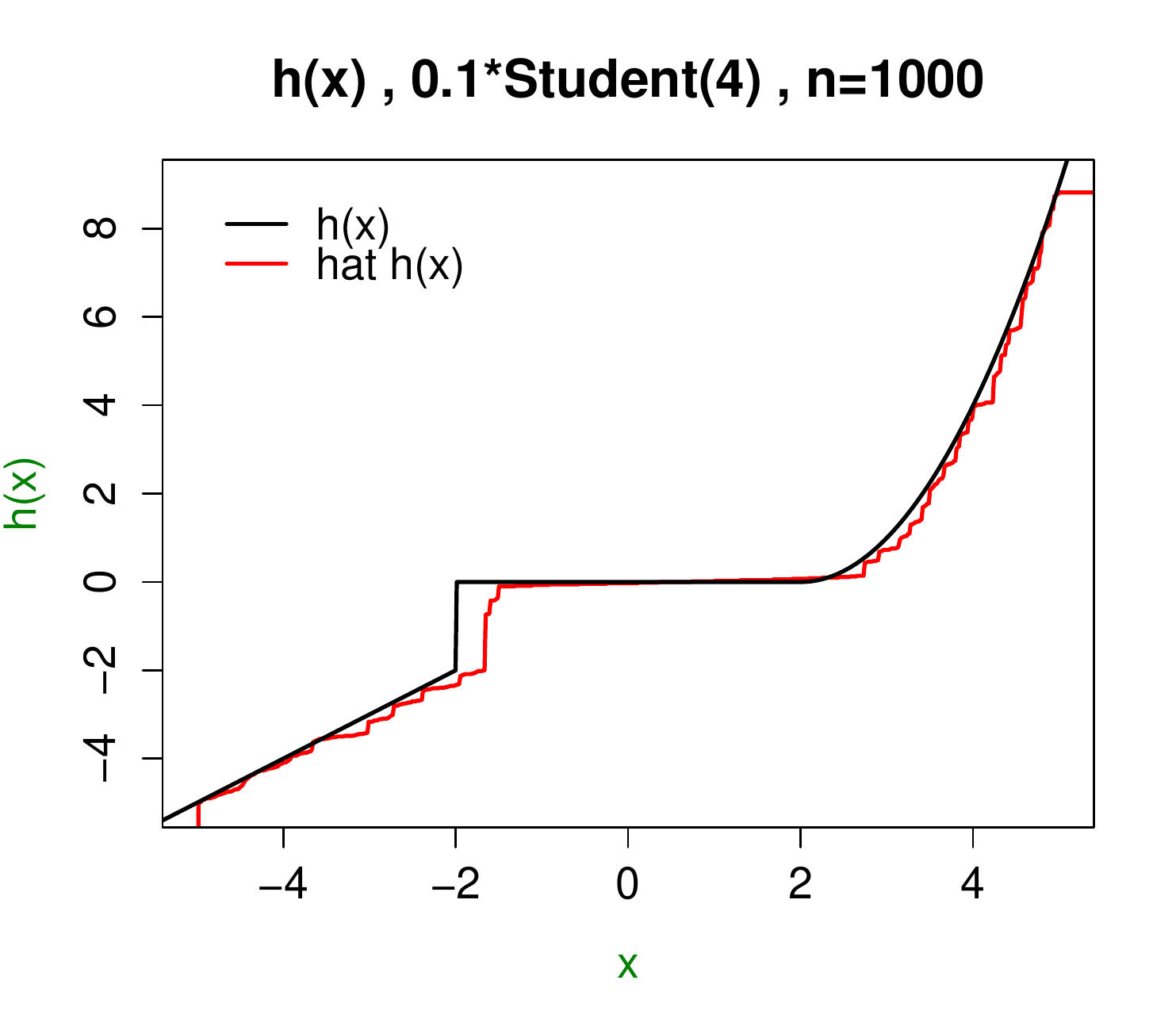}
\end{minipage}

\caption{A realisation of $\hat h$ (red) when $h$ is a function that has a discontinuity and a flat area, for different values of $m=n$, and different types of noise. Left to right: $\xi$ is (i) a $\mathcal N(0,0.1)$, (ii) a $\mathcal N(0,1)$, (iii) a $U([-0.5,0.5])$ and (iv) a $0.1 Student(4)$. Up to down: $m=n$ takes value (i) $100$, (ii) $500$ and (iii) $1000$.} \label{fig:exp2}
\end{figure*}
\end{center}

\subsubsection{Simulations in the case of deconvolution with the wrong noise $\xi$}

A third aspect that we illustrate with simulations is the impact of not knowing precisely the noise distribution $\xi$ (during the deconvolution process) on the estimator $\hat h$. More precisely, what is the impact to use a deconvolution distribution $\xi'$ that differs from $\xi$ in the deconvolution process for obtaining the estimate $\hat F_h$? We did experiments for two ``true" noises $\xi$, and four deconvolution distributions $\xi'$. The results are displayed in Figure~\ref{fig:exp3}.

The impact of not having a precise knowledge of $\xi$ is existent, but is not very important in these two examples. It implies that when confronted with a such problem, making a small error on the distribution of the noise $\xi$ is not completely altering the efficiency of the procedure.

\begin{figure*}[!htb]
\begin{center}
\begin{minipage}{7.4cm}
\includegraphics[width=7.4cm]{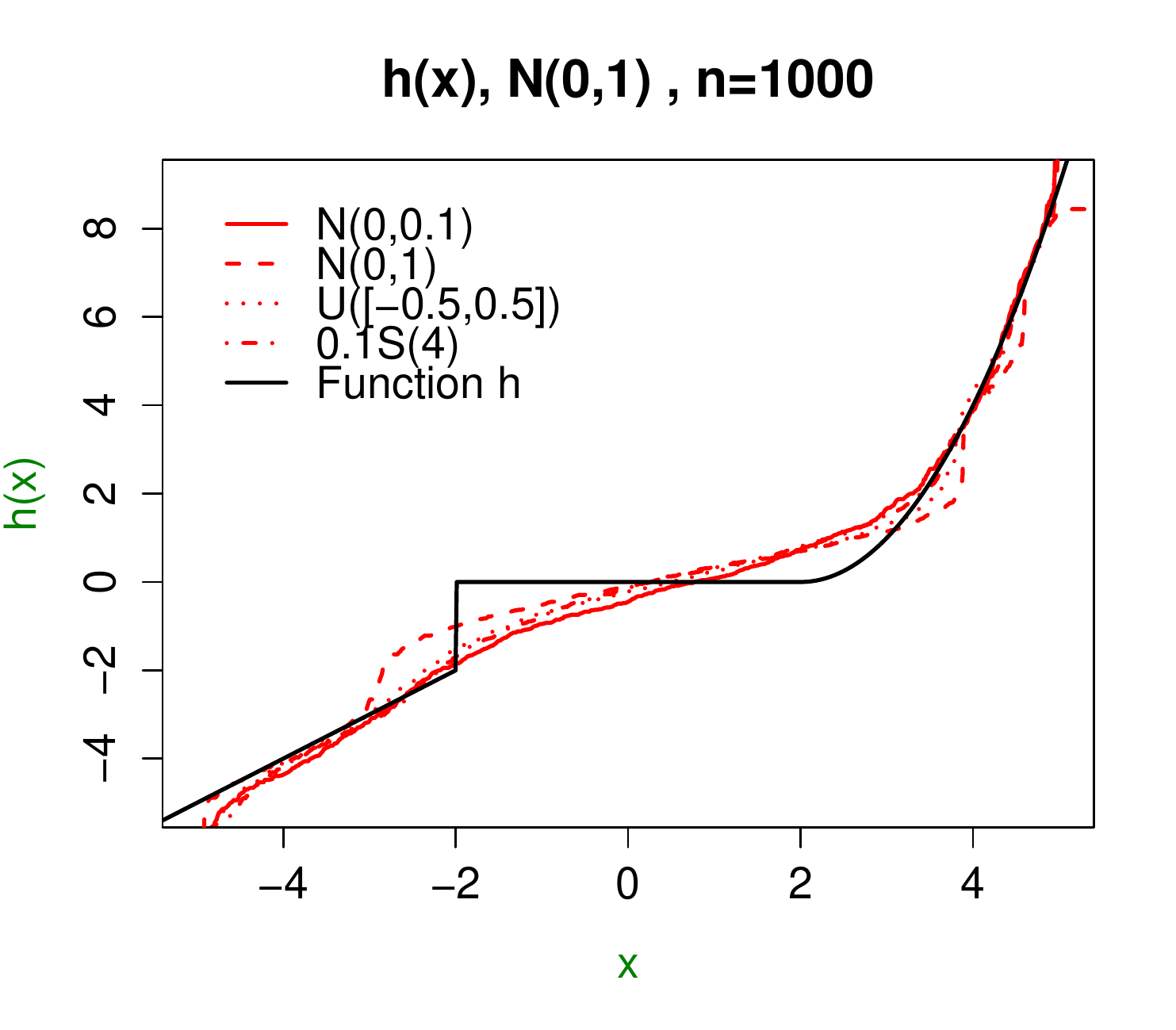}
\end{minipage}
\begin{minipage}{7.4cm}
\includegraphics[width=7.4cm]{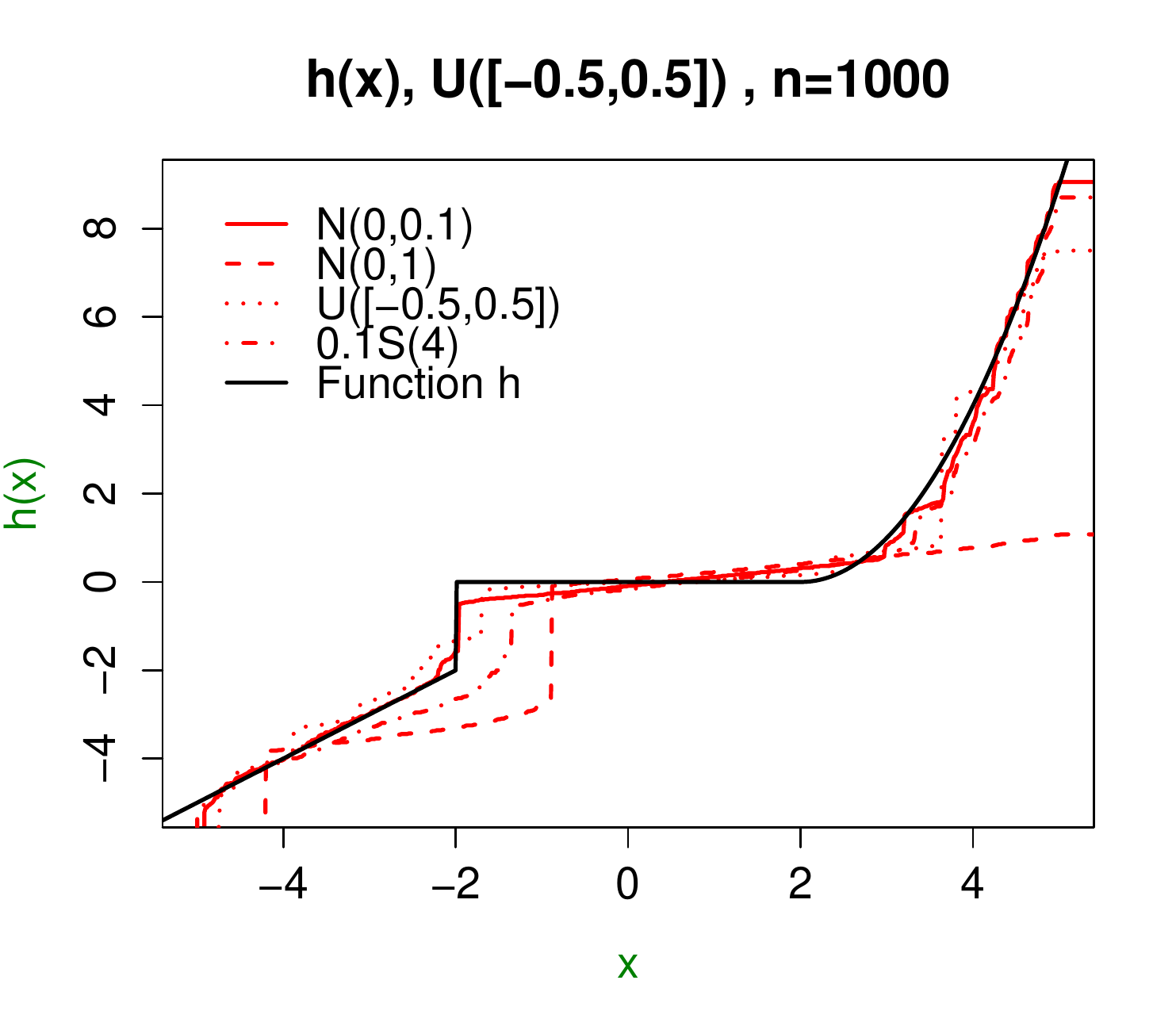}
\end{minipage}

\caption{A realisation of $\hat h$ (red) for $m=n=1000$ when deconvolving using different distributions of $\xi$ (the true $\xi$ is not available but one uses another value of $\xi$) when $h$ is a function that has a discontinuity and a flat area, for two types of ''true" noise. Left to right: $\xi$ is (i) a $\mathcal N(0,0.1)$, (ii) a $U([-0.5,0.5])$.} \label{fig:exp3}
\end{center}
\end{figure*}

\end{document}